\documentclass[11pt]{article}
\usepackage[margin=1in]{geometry}
\usepackage{lmodern}
\usepackage[round,authoryear]{natbib}
\usepackage[T1]{fontenc}
\usepackage[utf8]{inputenc}
\usepackage{amsmath,amssymb,amsthm,bm}
\usepackage{booktabs,tabularx,longtable}
\usepackage{graphicx}
\usepackage[font=small,labelfont=bf]{caption}
\usepackage{placeins}
\usepackage{needspace}
\usepackage{etoolbox}
\newtheorem{theorem}{Theorem}
\newtheorem{proposition}{Proposition}
\AtBeginEnvironment{theorem}{\par\noindent\begin{minipage}{\linewidth}}
\AtEndEnvironment{theorem}{\end{minipage}\par\medskip}
\AtBeginEnvironment{proposition}{\par\noindent\begin{minipage}{\linewidth}}
\AtEndEnvironment{proposition}{\end{minipage}\par\medskip}
\usepackage{algorithm}
\usepackage{flafter}
\usepackage{algpseudocode}
\usepackage{microtype}
\usepackage{hyperref}
\usepackage{bookmark}
\hypersetup{colorlinks=true,linkcolor=black,citecolor=black,urlcolor=black,
  pdftitle={Spectral Graph Neural Networks with Hermite Polynomials: A Comprehensive Study},
  pdfauthor={Shuang Wu}}
\DeclareMathOperator*{\argmin}{arg\,min}
\newcommand{\best}[1]{\ensuremath{\mathbf{#1}}}
\newcommand{\key}[1]{\ensuremath{\mathbf{#1}^{\dagger}}}
\newcommand{\HermitePoly}{\texttt{HermNet}}
\newcommand{\ChebNet}{\texttt{ChebNet}}

\newcommand{\ChebNetII}{\texttt{ChebNetII}}

\newcommand{\APPNP}{\texttt{APPNP}}
\newcommand{\GPRGNN}{\texttt{GPR-GNN}}
\newcommand{\BernNet}{\texttt{BernNet}}
\newcommand{\JacobiConv}{\texttt{JacobiConv}}
\newcommand{\FavardGNN}{\texttt{FavardGNN}}
\newcommand{\OptBasisGNN}{\texttt{OptBasisGNN}}

\title{Spectral Graph Neural Networks with Hermite Polynomials:\\A Comprehensive Study}
\author{Shuang Wu\\UCLA\\\texttt{shuangwu222@ucla.edu}}
\date{}

\begin{document}
\maketitle
\begin{abstract}
We study spectral graph neural networks built from Hermite polynomials and propose HermNet, a simple model that combines a nodewise predictor with normalized Hermite propagation. Its sparse recurrence requires neither eigendecomposition nor a learned basis. We distinguish the basic model from optional coordinate calibration, response normalization and Gaussian derivative regularization. Hermite and other complete polynomial bases span the same degree-bounded filter space, but their coordinates can produce different optimization behavior under limited training budgets. We analyze this behavior through spectral signal energy, label sampling, changes in learned features and the bias--variance trade-off of regularization. Controlled synthetic experiments identify a regime in which plain HermNet outperforms matched polynomial-basis alternatives, including with a jointly trained nonlinear predictor. Curvature regularization further improves HermNet when the same functional penalty is available to every comparator. Fixed-predictor controls support the advantage under short training budgets, but longer training removes the plain-model lead. Matched real-data comparisons show accuracy deficits, and architectural and numerical studies identify further limits. Together, the analysis and experiments clarify when Hermite propagation is useful and how calibration and regularization affect its performance.
\end{abstract}

\section{Introduction}
\label{sec:introduction}
Polynomial spectral graph neural networks learn a graph-frequency response through repeated sparse propagation. Different polynomial bases provide different coordinates for this response. Chebyshev, Bernstein and generalized PageRank models have made this approach practical across graph-learning tasks \citep{defferrard2016,chien2021,he2021,he2022}. A basic design question remains: when is a simple, fixed polynomial family a useful choice for learning a graph filter?

We propose \HermitePoly{}, a nodewise predictor followed by normalized probabilists' Hermite propagation. Algorithm~\ref{alg:filter} is the complete propagation mechanism. It uses unrestricted coefficients, a prescribed center and scale, and a three-term recurrence. Neither an eigendecomposition nor a learned basis is required. The plain model contains no reference calibration, response normalization or filter penalty. We introduce these operations separately as optional enhancements, so that their effects can be assessed independently of the proposed propagation.

Hermite polynomials offer a specific statistical structure. They are orthonormal under a Gaussian measure, and their derivatives have diagonal squared norms in the same coordinates. For graph filtering, the relevant measure weights each eigenvalue by the energy of the filtered signal. An approximately Gaussian eigenvalue histogram is therefore insufficient by itself. The usefulness of Hermite coordinates depends on spectral signal energy, polynomial degree, label coverage and the changing predictor. This connection motivates a model and testable conditions rather than a distribution-free performance claim.

Complete degree-$K$ bases span the same filter space, but their coordinates can change finite-step optimization. We therefore compare plain propagation and shared functional penalties separately. Exact minimizer sets correspond under an invertible basis change when the design and penalty are transformed consistently.

The paper makes three contributions:
\begin{enumerate}
\item \textbf{A simple model with explicit extensions.} We formulate HermNet for classification and regression, give its sparse propagation algorithm, and separate prescribed coordinates from data-dependent calibration and functional regularization.
\item \textbf{Properties that explain its scope.} We relate the Hermite response Gram to spectral signal energy, derive conditioning bounds under label sampling and feature changes, and separate regularization risk from the representation available to a fixed input. Further results establish a degree--scale constraint and a fixed-filter perturbation bound.
\item \textbf{A verified downstream regime.} Independent synthetic evaluations identify lower prediction error after five updates against five matched complete bases. The result holds with a jointly learned nonlinear predictor, where a common curvature penalty also improves HermNet. Real-data, longer-budget, native-model and numerical controls delimit the advantage.
\end{enumerate}

Earlier Hermite filtering applications are reviewed in Section~\ref{sec:related}. Our favorable experiment establishes an update-budget advantage, with native-model performance and total fitting cost assessed separately.

\section{Related Work}
\label{sec:related}
\paragraph{Polynomial filters and adaptive bases.}
ChebNet evaluates localized Chebyshev filters \citep{defferrard2016}, PPNP/APPNP separate prediction from propagation \citep{gasteiger2019}, and GPR-GNN learns signed propagation weights \citep{chien2021}. BernNet, ChebNetII and JacobiConv study alternative polynomial parameterizations \citep{he2021,he2022,wang2022}. FavardGNN and OptBasisGNN learn or adapt the basis \citep{guo2023}, while AdaptKry and UniFilter adapt propagation to the graph \citep{huang2024krylov,huang2024universal}. ClenshawGNN and GCNII introduce recurrence-inspired or residual architectures \citep{guo2023clenshaw,chen2020}, and FAGCN combines low- and high-frequency information \citep{bo2021}. HermNet uses a fixed Hermite recurrence within the same complete degree-bounded polynomial space. Its contribution concerns the learning procedure and the conditions under which those coordinates are useful.

\paragraph{Regularization and spectral extensions.}
NewtonNet learns interpolation values with a graph-dependent shape prior \citep{xu2024shape}. Filter smoothness also controls graph perturbation sensitivity \citep{gama2020}, although piecewise constant components can better represent sharp responses \citep{martirosyan2025}. Triple Filter Ensembles combine low- and high-pass families \citep{duan2024}, PolyGCL learns polynomial filters for contrastive views \citep{chen2024polygcl}, and SLOG extends inductive spectral learning beyond polynomials \citep{xu2024slog}. Our Gaussian derivative penalties are function-space preferences with a simple diagonal form in Hermite coordinates. Transforming them into every comparator basis permits controlled evaluation without changing the underlying preference.

\paragraph{Hermite filtering and evaluation.}
Hermite filters precede this work in cortical learning, LSAP diffusion and PolyCF recommendation \citep{huang2020,sim2024,qin2024}. Polynomial network constructions provide a complementary numerical perspective \citep{tang2024}. We study a simple supervised HermNet, its Gaussian spectral-energy connection and its finite-budget downstream behavior. Recent work questions whether conventional Fourier interpretations explain practical spectral-GNN gains \citep{guo2025position}, benchmarks effectiveness together with memory and efficiency \citep{liao2025}, and identifies weaknesses in heterophily evaluation \citep{platonov2023}. These findings motivate separate assessments of coordinate choice, regularization, architecture and total fitting cost. The literature reviewed here is broader than the experimental baseline roster.

\section{Simple HermNet}
\label{sec:method}
\label{sec:problem}
\subsection{Learning problem and spectral signal energy}
Let $G=(V,E,\mathbf A)$ be an undirected graph with $n$ nodes, symmetric nonnegative adjacency $\mathbf A$, features $\mathbf X\in\mathbb R^{n\times d}$, and normalized Laplacian
\begin{equation}
\mathbf L=\mathbf I_n-\mathbf D^{-1/2}\mathbf A\mathbf D^{-1/2}
=\mathbf U\boldsymbol\Lambda\mathbf U^\top,
\qquad 0\leq\lambda_i\leq2.
\label{eq:laplacian}
\end{equation}
Here $D_{ii}=\sum_j A_{ij}$ and $(D^{-1/2})_{ii}=0$ for an isolated node. Self-loops follow the declared graph construction and remain fixed within a matched comparison. Disjoint sets $\mathcal T,\mathcal V,\mathcal S$ specify training, validation and test nodes. The full graph and features are observed, training labels determine gradients, validation labels select configurations and checkpoints, and test targets enter evaluation only.

Let $f_{\boldsymbol\phi}:\mathbb R^d\to\mathbb R^C$ be a nodewise predictor. For classification, $C$ is the class count. For scalar regression, $C=1$. With $\mathbf H=f_{\boldsymbol\phi}(\mathbf X)$ and $g_c\in\mathcal P_K$, the polynomials of degree at most $K$, define
\begin{equation}
\mathbf z_c=g_c(\mathbf L)\mathbf h_c,
\qquad \mathbf Z=[\mathbf z_1,\ldots,\mathbf z_C],
\qquad P_{ic}=\frac{\exp Z_{ic}}{\sum_a\exp Z_{ia}}.
\label{eq:prediction}
\end{equation}
The supervised loss is
\begin{equation}
\mathcal L_{\mathcal T}=
\begin{cases}
-|\mathcal T|^{-1}\sum_{i\in\mathcal T}\log P_{iy_i}, &\text{classification},\\[2pt]
(2|\mathcal T|)^{-1}\sum_{i\in\mathcal T}(Z_i-y_i)^2, &\text{regression}.
\end{cases}
\label{eq:supervised_loss}
\end{equation}
The polynomial stage is linear conditional on $\mathbf H$. Learning $\boldsymbol\phi$ makes the complete predictor nonlinear. The identity predictor is used only in explicitly designated fixed-feature experiments.

Let $\mathbf R\in\mathbb R^{n\times d_R}$ be a nonzero matrix of node signals, with $d_R$ channels. It specifies the signal whose spectral energy is measured. Examples are the current predictor output $\mathbf R=\mathbf H$, the input $\mathbf R=\mathbf X$, or a frozen reference output $\mathbf R=f_{\boldsymbol\phi_0}(\mathbf X)$ used only for optional calibration. For coordinates $(\mu,s)$ with $s>0$, define
\begin{equation}
\nu_{\mathbf R}^{\mu,s}
=\sum_{i=1}^{n}w_i\,\delta_{(\lambda_i-\mu)/s},
\qquad w_i=\frac{\|\mathbf u_i^\top\mathbf R\|_2^2}{\|\mathbf R\|_F^2},
\qquad \sum_iw_i=1.
\label{eq:measure}
\end{equation}
Weights within a repeated eigenspace add, making the measure independent of the eigenbasis chosen there. Gaussian input entries do not imply a Gaussian measure in \eqref{eq:measure}. The latter describes energy over graph frequencies, not the marginal distribution of node features.

\begin{figure}[tbp]
\centering
\includegraphics[width=\linewidth]{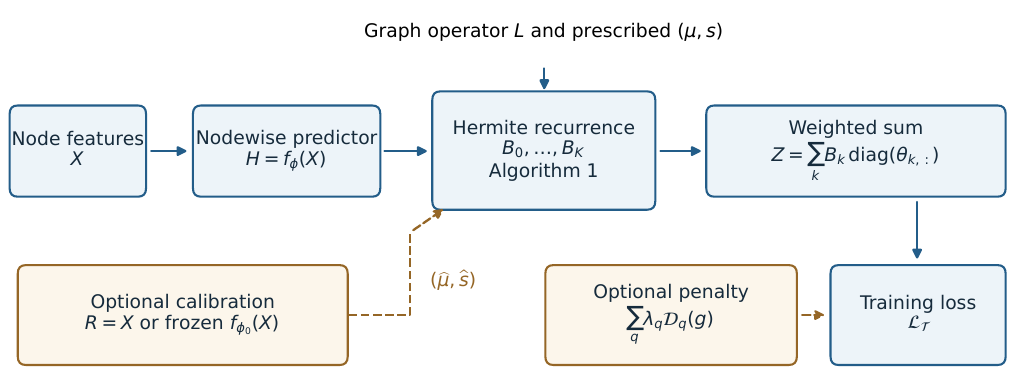}
\caption{\textbf{Simple HermNet and its optional extensions.} Solid arrows give the plain model: a nodewise predictor, Algorithm~\ref{alg:filter}, and unrestricted channelwise coefficients. Dashed arrows show optional coordinate calibration from a nonzero signal matrix $\mathbf R$ and a Gaussian functional penalty. Prescribed coordinates require no reference fit.}
\label{fig:overview}
\end{figure}

\subsection{Hermite propagation}
Define normalized probabilists' Hermite polynomials by
\begin{equation}
\begin{aligned}
h_0(z)&=1,& h_1(z)&=z,\\
h_{k+1}(z)&=\frac{z h_k(z)}{\sqrt{k+1}}-\sqrt{\frac{k}{k+1}}h_{k-1}(z),& k&\geq1.
\end{aligned}
\label{eq:hermite}
\end{equation}
For $W\sim\mathcal N(0,1)$, $\mathbb E[h_j(W)h_k(W)]=\delta_{jk}$ \citep{nistHermite}. These are polynomials without a Gaussian envelope. With $\mathbf S=(\mathbf L-\mu\mathbf I_n)/s$, HermNet uses
\begin{equation}
 g_c(\lambda)=\sum_{k=0}^{K}\theta_{kc}h_k\!\left(\frac{\lambda-\mu}{s}\right),
 \qquad \mathbf Z=\sum_{k=0}^{K}h_k(\mathbf S)\mathbf H\operatorname{diag}(\boldsymbol\theta_{k,:}).
\label{eq:filter}
\end{equation}
The coefficients $\boldsymbol\Theta\in\mathbb R^{(K+1)\times C}$ are unrestricted. The response can therefore be low-pass, high-pass or mixed. Algorithm~\ref{alg:filter} evaluates \eqref{eq:filter} using sparse products with $\mathbf L$.

\label{sec:basic}
The plain model trains \eqref{eq:filter} under \eqref{eq:supervised_loss} with no filter penalty. The fixed-coordinate version uses $(\mu,s)=(1,1)$. The prescribed-scale version selects $s$ from a finite menu fixed before observing an evaluation draw. Both use Algorithm~\ref{alg:filter} alone for propagation. Prescribed scales are hyperparameters, not estimates of spectral moments.

\begin{algorithm}[tbp]
\caption{Simple HermNet propagation}
\label{alg:filter}
\begin{algorithmic}[1]
\Require $\mathbf L$, $\mathbf H=f_{\boldsymbol\phi}(\mathbf X)$, $K\geq0$, $\mu\in\mathbb R$, $s>0$, $\boldsymbol\Theta$
\State $\mathbf B_0\gets\mathbf H$, \quad $\mathbf Z\gets\mathbf B_0\operatorname{diag}(\boldsymbol\theta_{0,:})$
\If{$K\geq1$}
  \State $\mathbf B_1\gets(\mathbf L\mathbf B_0-\mu\mathbf B_0)/s$
  \State $\mathbf Z\gets\mathbf Z+\mathbf B_1\operatorname{diag}(\boldsymbol\theta_{1,:})$
  \For{$k=1,\ldots,K-1$}
    \State $\mathbf B_{k+1}\gets\dfrac{\mathbf L\mathbf B_k-\mu\mathbf B_k}{s\sqrt{k+1}}-\sqrt{\dfrac{k}{k+1}}\mathbf B_{k-1}$
    \State $\mathbf Z\gets\mathbf Z+\mathbf B_{k+1}\operatorname{diag}(\boldsymbol\theta_{k+1,:})$
  \EndFor
\EndIf
\State \Return $\mathbf Z$
\end{algorithmic}
\end{algorithm}

\subsection{Function class, locality and cost}
\begin{proposition}[Complete polynomial propagation]
\label{prop:complete}
For any fixed $\mu\in\mathbb R$, $s>0$ and degree $K$, the functions $h_k((\lambda-\mu)/s)$, $0\leq k\leq K$, form a basis of $\mathcal P_K$. HermNet is $K$-hop local. Writing $m=\operatorname{nnz}(\mathbf L)+n$, its propagation costs $\mathcal O(KmC)$ for $K\geq1$, with $\mathcal O(nC)$ additional inference memory.
\end{proposition}
Appendix~\ref{proof:complete} proves the statement.
Standard reverse-mode training may retain $\mathcal O(KnC)$ activations. The cost excludes the nodewise predictor and optional calibration. The same asymptotic propagation order applies to several competing recurrences, so this proposition does not establish a runtime advantage.

\section{Properties, Regularization and Calibration}
\label{sec:theory}
\label{sec:enhancements}
All derivations appear in Appendix~\ref{app:proofs}. The results distinguish the coordinates used to fit a filter from its function class and regularization prior.
\subsection{Gaussian alignment and finite-step optimization}
For a fixed scalar signal $\mathbf r$, normalized by $\|\mathbf r\|_2^2=n$, let
\begin{equation}
\mathbf F=[h_0(\mathbf S)\mathbf r,\ldots,h_K(\mathbf S)\mathbf r],
\quad \mathbf G=\frac{\mathbf F^\top\mathbf F}{n},
\quad G_{jk}=\int h_j(z)h_k(z)\,d\nu_{\mathbf r}^{\mu,s}(z).
\label{eq:main_gram}
\end{equation}
If the polynomial inner products agree with a standard Gaussian through degree $2K$, then $\mathbf G=\mathbf I_{K+1}$. More generally, $\epsilon=\|\mathbf G-\mathbf I\|_2<1$ gives
\begin{equation}
1-\epsilon\leq\lambda_{\min}(\mathbf G)\leq\lambda_{\max}(\mathbf G)\leq1+\epsilon,
\qquad \kappa_2(\mathbf G)\leq\frac{1+\epsilon}{1-\epsilon}.
\label{eq:gram_condition}
\end{equation}
A finite graph need only approximate these finitely many inner products. Matching the first two moments or the fourth moment alone is insufficient.

A complete basis transformation $\boldsymbol\theta=\mathbf T\boldsymbol\beta$ preserves predictions. Ordinary gradient descent on $J(\mathbf T\boldsymbol\beta)$ induces the update
\begin{equation}
\boldsymbol\theta_{t+1}=\boldsymbol\theta_t-\alpha\mathbf T\mathbf T^\top\nabla J(\boldsymbol\theta_t).
\label{eq:main_coordinate_gd}
\end{equation}
The basis determines a positive-definite preconditioner for this fixed-transform gradient update. Equivalent exact solutions can therefore coexist with different finite-step predictions. This identity does not cover general adaptive optimizers or guarantee neural-model performance.

\subsection{Labeled nodes and changes in the predictor}
Full-node alignment need not survive label sampling or predictor updates. For the normalized signal in \eqref{eq:main_gram}, fix $q$ candidate coordinate pairs independently of the label mask. Let $0<\pi\leq1$ and $0<\delta<1$ be the sampling and failure probabilities. Write $d_K=K+1$, let $\mathbf f_{ji}^\top$ be row $i$ of $\mathbf F_j$, and define
\begin{equation}
\begin{gathered}
 B_j=\max_i\|\mathbf f_{ji}\|_2^2,\quad \epsilon_j=\|\mathbf G_j-\mathbf I_{d_K}\|_2,
 \quad t_\delta=\log\frac{2qd_K}{\delta},\quad 
 \eta_j=\sqrt{\frac{2B_j\|\mathbf G_j\|_2t_\delta}{\pi n}}+
 \frac{2B_jt_\delta}{3\pi n}.
\end{gathered}
\label{eq:main_mask_eta}
\end{equation}
\begin{theorem}[Label sampling and feature changes]
\label{thm:mask}
Condition on $\mathbf L$, $\mathbf r$ and the candidate pairs. Let $\mathbf M$ have independent diagonal entries $M_{ii}\sim\operatorname{Bernoulli}(\pi)$, $0<\pi\leq1$. For $0<\delta<1$, with probability at least $1-\delta$, simultaneously for every candidate $j$,
\begin{equation}
\left\|\frac{\mathbf F_j^\top\mathbf M\mathbf F_j}{\pi n}-\mathbf G_j\right\|_2\leq\eta_j.
\label{eq:main_mask}
\end{equation}
On this event, let $\widetilde{\mathbf F}_j=\mathbf F_j+\mathbf E_j$, allowing $\mathbf E_j$ to depend on $\mathbf M$, and put $\zeta_j=\|\mathbf M\mathbf E_j\|_2/\sqrt{\pi n}$. If $\epsilon_j+\eta_j<1$, then
\begin{equation}
\begin{gathered}
 a_j\mathbf I_{d_K}\preceq\frac{\widetilde{\mathbf F}_j^\top\mathbf M\widetilde{\mathbf F}_j}{\pi n}\preceq b_j\mathbf I_{d_K},\\
 a_j=\bigl(\sqrt{1-\epsilon_j-\eta_j}-\zeta_j\bigr)_+^2,\qquad
 b_j=\bigl(\sqrt{1+\epsilon_j+\eta_j}+\zeta_j\bigr)^2.
\end{gathered}
\label{eq:main_drift}
\end{equation}
Here $(x)_+=\max\{x,0\}$ and $\preceq$ is the positive-semidefinite order.
\end{theorem}
When $a_j>0$, fixed-feature quadratic gradient descent with Hessian normalized by $\pi n$ and step $1/b_j$ contracts coefficient error by at most $1-a_j/b_j$. Appendix~\ref{proof:mask} proves the theorem and this consequence.

A predictor fitted with the same labels is not independent of the mask. The probability statement requires a label-independent reference or a new fitting mask. The drift bound remains deterministic once initial singular-value bounds hold. Low row concentration and small feature changes thus matter alongside Gaussian alignment, without guaranteeing unrestricted end-to-end performance.

\subsection{Limits imposed by degree, scale and representation}
A graph spectrum is bounded, whereas the Gaussian reference is not. If $\nu$ is supported on $[-r,r]$ and its Hermite Gram through degree $K$ satisfies $\|\mathbf G-\mathbf I\|_2\leq\epsilon<1$, then
\begin{equation}
 r^2\geq(2K-1)\frac{1-\epsilon}{1+\epsilon},\qquad K\geq1.
\label{eq:main_support}
\end{equation}
Appendix~\ref{proof:support} establishes this necessary condition. In particular, $r\leq1$ and $\epsilon\leq0.2$ exclude $K\geq2$. Reducing $s$ expands the transformed interval, but may increase polynomial magnitudes and sensitivity.

For a fixed input and $N>0$ evaluation nodes, let $\mathcal U_K$ be the span of its degree-$K$ polynomial responses and $\boldsymbol\Pi_K$ the orthogonal projector onto that span. No full-rank assumption is required.
\begin{proposition}[Representation and fitting error]
\label{prop:representation}
For a noiseless target $\mathbf f\in\mathbb R^N$ and any fitted prediction $\widehat{\mathbf f}\in\mathcal U_K$,
\begin{equation}
\frac{\|\widehat{\mathbf f}-\mathbf f\|_2^2}{N}
=\underbrace{\frac{\|(\mathbf I-\boldsymbol\Pi_K)\mathbf f\|_2^2}{N}}_{\text{representation floor }B_K}
+\underbrace{\frac{\|\widehat{\mathbf f}-\boldsymbol\Pi_K\mathbf f\|_2^2}{N}}_{\text{within-space error }E_K}.
\label{eq:main_representation}
\end{equation}
An invertible complete-basis change preserves $\mathcal U_K$ and $B_K$. At fixed graph, input, target and evaluation nodes, $B_{K+1}\leq B_K$.
\end{proposition}
Appendix~\ref{proof:representation} gives the proof. Changing a learned predictor can change the space, whereas conditioning alone cannot remove a fixed-input representation floor.

\subsection{Gaussian derivative and curvature penalties}
For $g(\lambda)=\sum_k\theta_kh_k((\lambda-\mu)/s)$ and $\Lambda\sim\mathcal N(\mu,s^2)$, define for integer $q\geq0$
\begin{equation}
\mathcal D_q(g)=\mathbb E[(g^{(q)}(\Lambda))^2]
=\frac{1}{s^{2q}}\sum_{k=q}^{K}\frac{k!}{(k-q)!}\theta_k^2.
\label{eq:main_derivative}
\end{equation}
Empty sums are zero. Appendix~\ref{proof:derivative} proves the identity. Thus $q=0$ is response energy, $q=1$ is derivative energy, and $q=2$ is curvature energy. The regularized objective is
\begin{equation}
 \mathcal J=\mathcal L_{\mathcal T}+\lambda_\phi\|\boldsymbol\phi\|_2^2
 +\sum_{c=1}^{C}\sum_{q=0}^{2}\lambda_q\mathcal D_q(g_c),
 \qquad \lambda_q\geq0.
\label{eq:main_objective}
\end{equation}
These are preferences over functions. If $\boldsymbol\theta=\mathbf T_b\boldsymbol\beta_b$ represents the same filter in basis $b$ and $\mathbf W$ is the Hermite penalty matrix, then
\begin{equation}
 \boldsymbol\theta^\top\mathbf W\boldsymbol\theta
 =\boldsymbol\beta_b^\top\underbrace{\mathbf T_b^\top\mathbf W\mathbf T_b}_{\mathbf W_b}\boldsymbol\beta_b.
\label{eq:main_penalty_transform}
\end{equation}
All enhanced comparisons use the same functional penalty through this transformation. When the prior has fixed coordinates different from the propagation coordinates, $\mathbf T_b$ also converts to those prior coordinates. In particular, the nonlinear experiment uses the same $\mathcal N(1,1/8)$ curvature prior at every candidate propagation scale.

\subsection{Regularization risk and sensitivity}
For a fixed design $\mathbf F$, let $\mathbf M$ select $m_{\mathcal T}>0$ labeled nodes and define $\mathbf Q=\mathbf F^\top\mathbf M\mathbf F/m_{\mathcal T}$, $\mathbf G=\mathbf F^\top\mathbf F/n$ and
\begin{equation}
\widehat{\boldsymbol\theta}=\argmin_{\boldsymbol\theta}
\left\{\frac{\|\mathbf M(\mathbf F\boldsymbol\theta-\mathbf y)\|_2^2}{2m_{\mathcal T}}
+\boldsymbol\theta^\top\mathbf W\boldsymbol\theta\right\},\qquad
\mathbf A=\mathbf Q+2\mathbf W\succ0.
\label{eq:fixed_estimator}
\end{equation}
The matrix $\mathbf W\succeq0$ represents the functional penalty. Assume $\mathbf y=\mathbf F\boldsymbol\theta^\star+\boldsymbol\xi$, with conditional mean-zero noise and covariance $\sigma^2\mathbf I_n$. The design, mask, coordinates and penalty are fixed independently of this noise.
\begin{proposition}[Bias--variance decomposition]
\label{prop:risk}
The expected noiseless prediction error is
\begin{equation}
\begin{aligned}
\mathcal R(\mathbf W)&=\mathbb E\!\left[\frac{\|\mathbf F(\widehat{\boldsymbol\theta}-\boldsymbol\theta^\star)\|_2^2}{n}\right]
=\|\mathbf G^{1/2}\mathbf A^{-1}2\mathbf W\boldsymbol\theta^\star\|_2^2
+\frac{\sigma^2}{m_{\mathcal T}}\operatorname{tr}(\mathbf G\mathbf A^{-1}\mathbf Q\mathbf A^{-1}).
\end{aligned}
\label{eq:main_risk}
\end{equation}
If $\mathbf G=\mathbf Q=\mathbf I$ and $\mathbf W$ is diagonal, put $u_k=2W_{kk}$ and $v=\sigma^2/m_{\mathcal T}$. Then
\begin{equation}
\mathcal R(\mathbf W)-\mathcal R(\mathbf0)
=\sum_{k=0}^K\frac{u_k^2(\theta_k^\star)^2-v(2u_k+u_k^2)}{(1+u_k)^2}.
\label{eq:main_shrinkage}
\end{equation}
\end{proposition}
For $u_k>0$, shrinkage improves that coefficient's risk when $(\theta_k^\star)^2<v(1+2/u_k)$. Small high-order coefficients and noisy labels can therefore favor derivative shrinkage. A sharp response can incur too much bias. Appendix~\ref{proof:risk} also establishes invariance of exact fitted signals under \eqref{eq:main_penalty_transform}. A shared penalty can help prediction without conferring an exclusive statistical advantage on Hermite coordinates. Noise-dependent validation selection is outside this fixed-penalty identity.

For $K\geq1$, define $\rho_K(r)=\sup_{|z|\leq r}(\sum_{k=0}^{K-1}h_k(z)^2)^{1/2}$, $r=\max\{|\mu|,|2-\mu|\}/s$ and $M_g=\sup_{[0,2]}|g|$.
\begin{theorem}[Fixed-filter perturbations]
\label{thm:stability}
For symmetric operators $\mathbf L,\widetilde{\mathbf L}$ with spectra in $[0,2]$ and signals $\mathbf h,\widetilde{\mathbf h}\in\mathbb R^n$, the same filter coefficients and coordinates give
\begin{equation}
\begin{aligned}
\|g(\widetilde{\mathbf L})\widetilde{\mathbf h}-g(\mathbf L)\mathbf h\|_2
\leq{}&\rho_K(r)\sqrt{\mathcal D_1(g)}\,
\|\widetilde{\mathbf L}-\mathbf L\|_F\|\mathbf h\|_2
+M_g\|\widetilde{\mathbf h}-\mathbf h\|_2.
\end{aligned}
\label{eq:main_stability}
\end{equation}
The operators need not have the same eigenvectors.
\end{theorem}
Appendix~\ref{proof:stability} proves the bound. Retraining or recalibration requires an additional term. The graph norm here is Frobenius.

\subsection{Optional coordinate calibration}
For a nonzero input or frozen reference $\mathbf R$ and scale floor $s_{\min}>0$, moment coordinates require no eigendecomposition:
\begin{equation}
 \mu_0=\frac{\langle\mathbf R,\mathbf L\mathbf R\rangle_F}{\|\mathbf R\|_F^2},
 \qquad s_0=\max\!\left\{s_{\min},\frac{\|(\mathbf L-\mu_0\mathbf I)\mathbf R\|_F}{\|\mathbf R\|_F}\right\}.
\label{eq:main_moments}
\end{equation}
One can then choose a coordinate pair from a declared finite grid by minimizing the reference Gram discrepancy. Response normalization instead rescales basis-response columns, and empirical QR orthogonalizes them when they are independent. These procedures can benefit any basis and incur additional computation. They are absent from the primary plain-HermNet comparison. Appendix~\ref{app:methods} specifies their safeguards, initialization transforms and costs.

\section{Experiments}
\label{sec:experiment}
\subsection{Research questions and evaluation design}
\label{sec:questions}
The experiments address four questions, in the order of the proposed model, its extensions, its mechanism and its empirical scope:
\begin{enumerate}
\item[\textbf{RQ1.}] \textbf{Plain propagation.} Does simple HermNet improve downstream prediction under a specified update budget? Sections~\ref{sec:main_learned}--\ref{sec:main_fixed} compare it with five complete polynomial bases.
\item[\textbf{RQ2.}] \textbf{Gaussian regularization.} Does the same functional penalty improve HermNet and preserve its advantage over equally enhanced bases? Section~\ref{sec:main_learned} evaluates both comparisons.
\item[\textbf{RQ3.}] \textbf{Gaussian connection.} When do spectral signal energy and Hermite conditioning explain the favorable regime? Section~\ref{sec:main_gaussian} examines alignment and adverse controls.
\item[\textbf{RQ4.}] \textbf{Benchmark transfer.} How does HermNet compare on real data and against native architectures? Section~\ref{sec:main_real} evaluates this broader scope.
\end{enumerate}

\paragraph{Comparison policy.}
To isolate the basis, competing models share the predictor, degree, coordinate menu, optimizer rule, initialization in function space and validation criterion. The five controls are Chebyshev, unrestricted Bernstein, symmetric Jacobi, Legendre and signed monomial expansions. These controls differ from native published architectures such as \BernNet{} or \JacobiConv{}, which can impose additional constraints or use different parameterizations.

A common initial filter must also be represented in each basis. Classification comparisons initialize $g_c\equiv1$. Degree-graded bases use coefficient vector $(1,0,\ldots,0)$, while the standard degree-$K$ Bernstein basis uses all ones. The downstream regression comparisons initialize $g_c\equiv0$. A shared predictor initialization then makes the initial prediction identical across bases. Predictor regularization and dropout are held common when present and are reported separately from filter enhancements.

\paragraph{Evaluation and uncertainty.}
Synthetic studies pair graph, feature, noise and split draws across methods. Means and sample standard deviations describe draw variability. Positive one-sided paired-$t$ lower bounds resolve mean improvements under the stated familywise adjustment. Each independent follow-up fixes its protocol before generating new draws. The broader research program contains adaptive exploration and is not a single preregistered experiment. Appendices~\ref{app:protocol}--\ref{app:results} report protocols, complete comparisons and adverse outcomes.

\subsection{Downstream prediction with a jointly learned predictor}
\label{sec:main_learned}
The central experiment uses a weighted eight-dimensional Boolean product graph with 256 nodes and four independent Gaussian input features. With independent positive uniform edge-dimension weights $w_\ell\sim\operatorname{Uniform}(0.75,1.25)$, the spectrum is approximately Gaussian after centering at one and scaling. The nonlinear teacher and response are
\begin{equation}
 u_i=\tanh(x_{i0}+0.5x_{i1})+0.3x_{i2}x_{i3},\qquad
 g_\star(\lambda)=\sin\!\left(0.75\frac{\lambda-1}{s_w}\right),
 \quad s_w=\frac{\|\mathbf w\|_2}{\sum_\ell w_\ell}.
\label{eq:teacher_main}
\end{equation}
We normalize $\mathbf f=g_\star(\mathbf L)\mathbf u$ to unit RMS and add observation noise with SD $0.3$. The teacher is not a Hermite polynomial. There are 32 training, 64 validation and 96 test nodes. Test MSE is measured against the clean target. The learner jointly trains a $4\!\to\!16\!\to\!1$ tanh predictor and an unrestricted degree-four filter.

Every basis uses center $1$ and the same eight prescribed scales
$\{0.2,0.3,1/\sqrt8,0.4,0.5,0.6,0.8,1\}$, shared predictor initialization, zero initial filter, and common optimizer menus. Predictor Adam rates are $\{0.01,0.03,0.1\}$. Filter gradient steps are a fraction in $\{0.5,1\}$ of the inverse current regularized response-Gram largest eigenvalue. The plain arm has no calibration, normalization, dropout or filter decay. The enhanced arm adds $(\tau_2/2)\mathbb E[g''(\Lambda)^2]$ with $\Lambda\sim\mathcal N(1,1/8)$ and $\tau_2\in\{0,0.001,0.01,0.1\}$, identically transformed in all bases. Validation selects the configuration and checkpoint from 48 plain or 192 enhanced candidates per basis.

\begin{figure}[tbp]
\centering
\includegraphics[width=\linewidth]{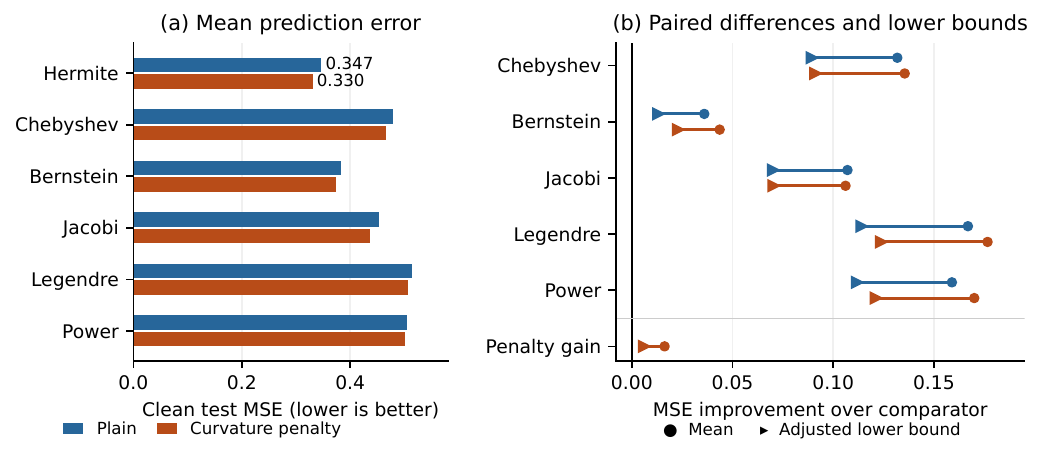}
\caption{\textbf{Independent five-update evaluation, 80 paired draws.} Left: mean clean test MSE for each basis, with a jointly learned predictor. Right: comparator-minus-Hermite mean differences (circles) and Bonferroni-adjusted one-sided lower bounds (triangles), connected for readability. The segments are not two-sided confidence intervals. The final row compares plain with enhanced Hermite. All eleven adjusted lower bounds exceed zero. Tables~\ref{tab:learned_five} and~\ref{tab:learned_five_bounds} give standard deviations and numerical bounds.}
\label{fig:main_learned}
\end{figure}

An earlier experiment fixed a primary budget of 20 updates and evaluated 40 fresh draws. Its secondary five-update results motivated the present independent experiment. Before generating 80 new draws, we fixed a five-update maximum, the sample size and all other settings. Figure~\ref{fig:main_learned} reports this independent result. Plain Hermite achieves MSE $0.34671$, compared with $0.38266$ for Bernstein, the best plain comparator: a \textbf{9.4\% reduction}. Enhanced Hermite achieves $0.33048$, compared with $0.37405$ for the best enhanced comparator: an \textbf{11.6\% reduction}. The enhancement improves Hermite itself by \textbf{4.7\%}.

All eleven planned contrasts have positive one-sided lower bounds after Bonferroni adjustment at familywise level $0.05$ within this experiment. The absolute lower bounds against Bernstein are $0.01313$ in the plain arm and $0.02297$ in the enhanced arm. For the incremental Hermite improvement, the lower bound is $0.00608$. A nonzero penalty is selected on 38 of 80 draws, and selected predictor parameters change from initialization in every Hermite fit. The paired-bootstrap sensitivity analysis agrees in sign. These results answer RQ1 and RQ2 for this declared architecture, task and update budget.

At the earlier primary 20-update endpoint, neither arm resolves superiority over every rival, although the incremental Hermite penalty benefit is resolved (adjusted lower bound $0.00702$). At 100 updates, plain Hermite loses its lead. Appendix~\ref{app:learned_product} retains that complete study alongside the independent follow-up.

\subsection{Fixed-predictor controls and the update budget}
\label{sec:main_fixed}
A separate regression experiment isolates polynomial fitting with an identity predictor on weighted ten-dimensional Boolean product graphs. Each graph has 1,024 nodes, with 80 training, 100 validation and 844 test nodes. A sine response defines the clean target, degree is four, and all filters begin at zero. The initial two-scale experiment gives Hermite MSE $0.00670$ versus $0.05966$ for Jacobi after five updates across 40 independent draws. Empirical QR gives $0.00657$, showing that Hermite is not superior to this adapted conditioning control.

To assess scale-search sensitivity, a further 40 independent draws use the common eight-scale menu $\{0.2,0.3,1/\sqrt{10},0.4,0.5,0.6,0.8,1\}$. Plain Hermite achieves $0.00511$, compared with Bernstein's $0.02165$, a \textbf{76.4\% reduction}. Allowing each basis to select response normalization gives Bernstein $0.02107$ and leaves the selected Hermite mean unchanged. All ten adjusted lower bounds, covering five rivals in each menu, remain positive. The smallest is $0.01027$.

\begin{figure}[tbp]
\centering
\includegraphics[width=\linewidth]{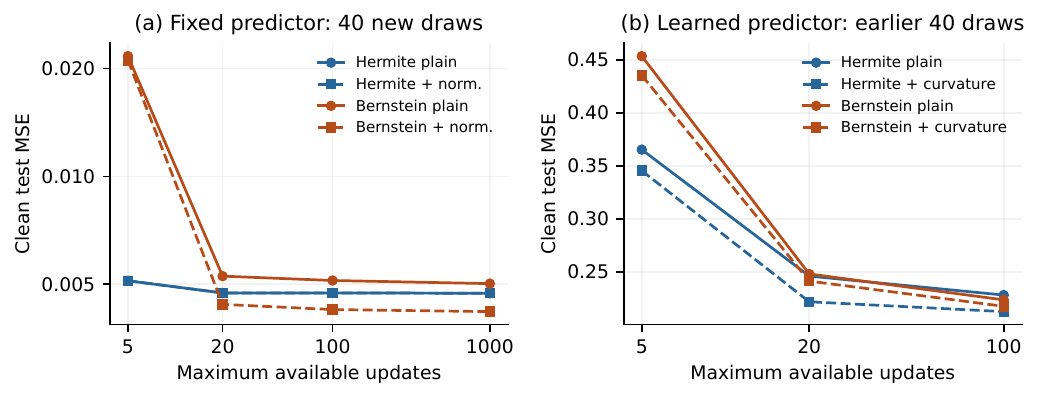}
\caption{\textbf{The advantage depends on the update budget.} Mean clean test MSE after validation selection within each available budget. Left: 40 independent fixed-predictor draws with the broad prescribed-scale menu. Solid lines use plain responses, dashed lines allow response normalization. Right: the earlier 40-draw nonlinear-predictor study, whose primary budget is 20 updates. Lines connect evaluated budgets and are not individual training trajectories. Bernstein is the closest comparator at five updates in both studies. Complete results for all bases appear in Figures~\ref{fig:broad_budgets} and~\ref{fig:learned_budgets}.}
\label{fig:main_budget}
\end{figure}

Figure~\ref{fig:main_budget} places the five-update result in context. In the broad-scale fixed-predictor experiment, the resolved lead disappears with more updates. With normalization available, Bernstein has a lower mean from 20 updates onward. Equivalent exactly solved filters agree to numerical precision in the corresponding coordinate checks. The favorable result concerns early prediction under a common optimizer rule, not an intrinsically better polynomial approximation space.

\subsection{Gaussian connection and diagnostic controls}
\label{sec:main_gaussian}
The product-graph construction gives a concrete connection to Gaussian coordinates. If $\varepsilon_\ell\in\{-1,1\}$ are independent equiprobable signs, the unweighted standardized eigenvalue distribution satisfies
\begin{equation}
 Z\ \overset{d}{=}\ \frac{\sum_\ell w_\ell\varepsilon_\ell}{\sqrt{\sum_\ell w_\ell^2}},
 \qquad \mathbb E[Z^2]=1,\qquad
 \mathbb E[Z^4]=3-2\frac{\sum_\ell w_\ell^4}{(\sum_\ell w_\ell^2)^2}.
\label{eq:product_moments}
\end{equation}
For balanced weights and increasing dimension, no coordinate dominates and the weighted sign sum approaches Gaussian behavior. The teacher's energy-weighted fourth moment in the independent nonlinear study averages $2.75$, compared with the Gaussian value $3$. This observation is consistent with the mechanism in \eqref{eq:main_gram}, although a single moment does not establish Gaussian Gram alignment through degree eight.

A dominant-coordinate intervention in the fixed-predictor study reduces the spectral fourth moment from approximately $2.77$ to $1.47$. Hermite then gives MSE $0.01467$, compared with Jacobi's $0.00694$. The intervention changes spectral scale and task difficulty as well as shape, so it does not isolate a causal effect of Gaussianity. It establishes a relevant adverse regime rather than a distributional identification result.

The separate fixed-design study provides a more direct conditioning diagnostic. Under narrow Gaussian-shaped signal energy and degree two, Hermite requires median gradient-descent counts of six on stochastic-block graphs and eight on geometric graphs, compared with five and seven for empirical whitening. Reference mismatch, uniform energy and higher degree weaken this behavior. Appendix~\ref{app:coordinate_tables} gives the complete factorial. In the representation analysis, input mismatch accounts for most error through $B_K$ in \eqref{eq:main_representation}, which no complete coordinate change can remove.

These findings answer RQ3 conditionally: Gaussian-weighted inner products explain why Hermite coordinates can be favorable, while sampling, predictor drift, degree and input representation restrict that mechanism. They do not prove that Gaussian feature noise alone creates a performance advantage.

\subsection{Real-data performance and native architectures}
\label{sec:main_real}
Table~\ref{tab:main_real} evaluates the plain fixed-coordinate model with the same width-32 ReLU predictor, classwise filter, dropout, degree menu, optimizer settings and validation rule in all six bases. There is no reference calibration, response normalization or filter decay. Two seeds select configurations and three final seeds evaluate them on one fixed split per dataset. All procedures use at most 400 updates, checking validation every five updates. Appendix~\ref{app:plain_real} specifies the complete protocol.

\begin{table}[H]
\centering
\caption{\textbf{Plain polynomial propagation on citation benchmarks.} Test accuracy (\%), mean $\pm$ sample SD over three final seeds on one fixed split. Coordinates are $(1,1)$ for all bases. Bold marks the largest mean in each column, not a significance claim.}
\label{tab:main_real}
\small
\begin{tabular}{@{}lccc@{}}
\toprule
Basis & Cora & CiteSeer & PubMed\\
\midrule
Hermite & $78.20\pm0.87$ & $69.63\pm0.99$ & $77.97\pm0.42$\\
Chebyshev & $\mathbf{81.57}\pm0.74$ & $\mathbf{71.70}\pm0.40$ & $80.03\pm0.84$\\
Bernstein & $80.83\pm0.81$ & $71.07\pm0.90$ & $78.57\pm0.29$\\
Jacobi & $79.37\pm2.15$ & $71.50\pm0.46$ & $80.27\pm0.71$\\
Legendre & $81.17\pm0.76$ & $71.13\pm0.99$ & $\mathbf{80.60}\pm0.70$\\
Power & $80.83\pm1.22$ & $70.93\pm0.81$ & $79.70\pm0.87$\\
\bottomrule
\end{tabular}
\end{table}

Plain Hermite trails the best means by $3.37$, $2.07$ and $2.63$ percentage points. The results show usable benchmark performance with consistent deficits. Three seeds on a single split do not establish noninferiority. A separate calibrated native-model comparison gives HermNet accuracies $83.80$, $71.30$ and $74.83$, while JacobiConv gives $83.90$, $72.13$ and $76.53$. That experiment uses a different protocol and an enhanced model. Its table must not be interpreted as evidence for plain Algorithm~\ref{alg:filter}.

Transfer to other architectures is mixed. Enhanced Hermite leads the common-basis ring-classification comparison, but native JacobiConv reaches $80.59\%$ versus Hermite's $73.90\%$. A subsequent shared-dropout comparison removes the Hermite lead. Channel-specific filters help several bases, and expanded real-data comparisons include larger Hermite deficits. Appendix~\ref{app:classification} reports these studies together with native-model settings and unsuccessful variants. Thus RQ4 is answered by a bounded empirical comparison, not a claim of state-of-the-art benchmark performance.

\subsection{Scope of the empirical findings}
\label{sec:main_scope}
The independent nonlinear experiment resolves both a five-update plain-Hermite advantage and its incremental curvature benefit. Fixed-predictor derivative and curvature studies do not resolve the incremental effect. Odd-hop restrictions benefit several families, while a bounded Gaussian prior meets none of eight prespecified practical-improvement criteria.

The favorable learned-predictor study lacks a full native-model panel and a wall-clock comparison. In the measured CPU pipeline, QR/Hermite total-time ratios are $0.9311$--$0.9964$. Thus fewer updates do not establish lower total fitting cost. Appendix~\ref{app:results} reports the precision, sparse-term and timing controls, including adverse outcomes.

\label{main:last}
\clearpage
\phantomsection
\label{references:start}
\bibliographystyle{plainnat}
{\small\raggedright\setlength{\bibsep}{1.5pt}\bibliography{references}}
\clearpage
\appendix
\makeatletter
\setlength{\@fptop}{0pt}
\setlength{\@fpsep}{8pt}
\setlength{\@fpbot}{0pt plus 1fil}
\makeatother
\pretocmd{\subsection}{\Needspace{6\baselineskip}}{}{}
\pretocmd{\subsubsection}{\Needspace{6\baselineskip}}{}{}
\section{Implementation Details}
\label{app:methods}
\subsection{Coordinate calibration and response transforms}
The optional trained-reference variant fits a feature-only predictor on training labels and selects its checkpoint on validation labels. Its frozen, nonzero output is $\mathbf R=f_{\boldsymbol\phi_0}(\mathbf X)$. Final predictor parameters are initialized independently. A zero or nonfinite reference records a failed calibrated fit. Input-only calibration instead uses $\mathbf R=\mathbf X$ and requires no reference training.

The moment coordinates in \eqref{eq:main_moments} use the Frobenius inner product $\langle A,B\rangle_F=\operatorname{tr}(A^\top B)$ and a positive scale floor $s_{\min}$.
For each candidate, compute $\mathbf Q_k=h_k((\mathbf L-\mu\mathbf I)/s)\mathbf R$ and
\begin{equation}
 [\mathbf G_K]_{jk}=\frac{\langle\mathbf Q_j,\mathbf Q_k\rangle_F}{\|\mathbf R\|_F^2},
 \quad\varepsilon_K=\|\mathbf G_K-\mathbf I\|_2,
 \quad(\widehat\mu_K,\widehat s_K)\in\argmin_{(\mu,s)\in\mathcal C}\varepsilon_K.
\label{eq:calibration}
\end{equation}
Here $\mathcal C$ contains admissible pairs $(\mu_0+a s_0,b s_0)$ from declared finite grids and $(1,1)$. The constraints are $s\geq s_{\min}$ and $\max\{|\mu|,|2-\mu|\}/s\leq r_{\max}$, with $s_{\min}\leq1$ and $r_{\max}\geq1$. Ties follow fixed candidate order. Coordinates are frozen during final training. The guarded moment variant uses $(\mu_0,s_0)$ if admissible and $(1,1)$ otherwise. Input-moment experiments omit the radius rejection where explicitly stated.

Response normalization replaces $\mathbf Q_k$ by $\mathbf Q_k/\max\{a_k,\epsilon\}$, where $a_k=\|\mathbf Q_k\|_F/\|\mathbf R\|_F$. The channel study uses $\epsilon=10^{-12}$. Full QR orthogonalizes flattened order-response columns when independent. Initial filter coefficients and functional penalties are transformed consistently. Neither operation changes the complete polynomial space.

\subsection{Functional priors and architectural controls}
The objective is \eqref{eq:main_objective}. With fixed prior coordinates, the Hermite penalty matrix is
\begin{equation}
\mathbf W=\lambda_0\mathbf I+\lambda_1s^{-2}\operatorname{diag}(k)_{k=0}^K
+\lambda_2s^{-4}\operatorname{diag}(k(k-1))_{k=0}^K.
\label{eq:objective}
\end{equation}
All complete bases use \eqref{eq:main_penalty_transform}. For a probability measure $\omega$ and basis vector $\mathbf b(\lambda)$,
\begin{equation}
 \mathbf J_\omega=\int\mathbf b'(\lambda)\mathbf b'(\lambda)^\top\,d\omega(\lambda),
 \qquad D_\omega(g)=\boldsymbol\theta^\top\mathbf J_\omega\boldsymbol\theta.
\label{eq:general_prior}
\end{equation}
Let $\gamma=\mathcal N(\mu,s^2)$, $I=[0,2]$ and $a=\gamma(I)$. With $\gamma_I=\gamma(\cdot\mid I)$ and $\gamma_{I^c}=\gamma(\cdot\mid I^c)$,
\begin{equation}
 D_\gamma(g)=aD_{\gamma_I}(g)+(1-a)D_{\gamma_{I^c}}(g).
\label{eq:prior_support}
\end{equation}
Replacing the full Gaussian by a bounded measure changes the functional prior. Its derivative matrix generally becomes nondiagonal. Prior comparisons specify measure normalization and include zero strength.

The odd-hop control restricts every basis to the same odd polynomial subspace on bipartite graphs. It changes the function class, unlike a coordinate transform. The channel-filter control uses
\begin{equation}
 \mathbf Z=\sum_{k=0}^K b_k(\mathbf S)\mathbf X\mathbf A_k+\mathbf1_n\mathbf b^\top,
 \qquad \mathbf A_k\in\mathbb R^{d\times C}.
\label{eq:channels}
\end{equation}
Its factorized counterpart sets $\mathbf A_k=\mathbf W_0\operatorname{diag}(\boldsymbol\theta_{k,:})$ and retains the bias. For each output channel $c$,
\[
 [\mathbf A_0(:,c),\ldots,\mathbf A_K(:,c)]
 =\mathbf W_0(:,c)\boldsymbol\theta_{:,c}^{\top},\qquad
 \operatorname{rank}[\mathbf A_0(:,c),\ldots,\mathbf A_K(:,c)]\leq1.
\]
The restriction couples orders and input channels, rather than requiring each $\mathbf A_k$ to have rank one. Unrestricted $\mathbf A_k$ enlarge the architecture for every complete basis \citep{defferrard2016}.

\subsection{Training and computational cost}
\label{sec:training}
For each declared configuration: (i) construct any optional reference and coordinate transform, (ii) initialize the common predictor and map the initial filter into its coefficients, (iii) train using only training labels, (iv) retain checkpoints using the declared validation rule, and (v) restore the selected configuration and checkpoint before test evaluation. Classification starts from $g_c\equiv1$, whereas the product-regression studies start from $g\equiv0$. Warm-start studies are identified separately. Degree, penalties, optimizer, allocation and tie rules are fixed within each comparison. No global convergence is assumed.

Writing $m=\operatorname{nnz}(\mathbf L)+n$, propagation costs $\mathcal O(KmC)$, with $\mathcal O(nC)$ additional inference memory and up to $\mathcal O(KnC)$ retained training activations. For $q$ candidates and a $d_R$-channel reference, explicit Gram calibration costs $\mathcal O(q[Km d_R+K^2n d_R+K^3])$, in addition to reference training. Timings include these costs where applicable.

\FloatBarrier
\section{Proofs}
\label{app:proofs}
This appendix gives the proofs of the statements in Sections~\ref{sec:method}--\ref{sec:theory} and the product-graph identity in Section~\ref{sec:main_gaussian}.

\subsection{Polynomial completeness and coordinate updates}
\label{proof:complete}
\begin{proof}[Proof of Proposition~\ref{prop:complete}]
The leading coefficient of $h_k((\lambda-\mu)/s)$ is $1/(s^k\sqrt{k!})\ne0$. The monomial change-of-basis matrix is triangular with nonzero diagonal, so these functions span $\mathcal P_K$. Hence
\[
 g(\mathbf L)=\sum_{k=0}^Ka_k\mathbf L^k,\qquad
 (\mathbf L^k)_{ij}\ne0\ \Longrightarrow\ \operatorname{dist}(i,j)\leq k.
\]
This proves $K$-hop locality. Each recurrence step uses one sparse multiplication and $\mathcal O(nC)$ elementwise work. Two consecutive responses and the accumulator give the stated inference memory.
\end{proof}
\begin{proof}[Proof of \eqref{eq:main_coordinate_gd} and \eqref{eq:gram_condition}]

For a fixed invertible $\boldsymbol\theta=\mathbf T\boldsymbol\beta$, the chain rule yields
\begin{equation}
\begin{aligned}
\nabla_\beta J(\mathbf T\boldsymbol\beta)&=\mathbf T^\top\nabla_\theta J(\boldsymbol\theta),\\
\mathbf T\boldsymbol\beta_{t+1}&=\mathbf T\boldsymbol\beta_t-\alpha\mathbf T\mathbf T^\top\nabla_\theta J(\boldsymbol\theta_t).
\end{aligned}
\label{eq:coordinate_gd}
\end{equation}
This proves \eqref{eq:main_coordinate_gd} for ordinary gradient descent. It does not cover adaptive optimizers or a changing transform. The spectral theorem gives the Gram identity \eqref{eq:main_gram}. Gaussian orthonormality gives $\mathbf G=\mathbf I$ when the corresponding inner products match. The norm bound then gives \eqref{eq:gram_condition} by eigenvalue perturbation.
\end{proof}

\subsection{Proof of the Gaussian derivative identity}
\label{proof:derivative}
\begin{proof}[Proof of \eqref{eq:main_derivative} and \eqref{eq:main_penalty_transform}]
The identity $h_k'=\sqrt{k}h_{k-1}$ \citep{nistHermite} gives, for $0\leq q\leq K$,
\[
 g^{(q)}(\lambda)=s^{-q}\sum_{k=q}^K\theta_k\sqrt{\frac{k!}{(k-q)!}}
 h_{k-q}\!\left(\frac{\lambda-\mu}{s}\right).
\]
With $\Lambda\sim\mathcal N(\mu,s^2)$, orthonormality eliminates cross terms:
\begin{equation}
\begin{aligned}
\mathbb E[g^{(q)}(\Lambda)^2]
&=s^{-2q}\sum_{j,k=q}^K\theta_j\theta_k
 \sqrt{\frac{j!k!}{(j-q)!(k-q)!}}\,\delta_{jk}\\
&=s^{-2q}\sum_{k=q}^K\frac{k!}{(k-q)!}\theta_k^2.
\end{aligned}
\label{eq:energy}
\end{equation}
For $q>K$, both the derivative and the empty sum vanish. This proves \eqref{eq:main_derivative}, including energy, derivative and curvature penalties. Substitution of $\boldsymbol\theta=\mathbf T_b\boldsymbol\beta_b$ proves \eqref{eq:main_penalty_transform}. \end{proof}

\subsection{Proof of the label-sampling bound}
\label{proof:mask}
\begin{proof}[Proof of Theorem~\ref{thm:mask}]
Let $d_K=K+1$ and use the notation in the statement.

\noindent\textbf{1. Center the sampled Gram.}  Fix a candidate $j$ and suppress its index. Define the independent, mean-zero, self-adjoint matrices
\[
\mathbf X_i=\frac{M_{ii}-\pi}{\pi n}\mathbf f_i\mathbf f_i^\top,
\qquad
\sum_{i=1}^n\mathbf X_i=\frac{\mathbf F^\top\mathbf M\mathbf F}{\pi n}-\mathbf G.
\]
\noindent\textbf{2. Bound the summands and variance.}  Since $|M_{ii}-\pi|\leq1$ and $\|\mathbf f_i\|_2^2\leq B$,
\[
\|\mathbf X_i\|_2\leq\frac{B}{\pi n},\qquad
\sum_i\mathbb E[\mathbf X_i^2]
=\frac{1-\pi}{\pi n^2}\sum_i\|\mathbf f_i\|_2^2\mathbf f_i\mathbf f_i^\top
\preceq\frac{B}{\pi n}\mathbf G.
\]
\noindent\textbf{3. Apply concentration uniformly over coordinates.}  Applying the matrix Bernstein inequality of \citet[Theorem~1.4]{tropp2012} to both signs yields, for every $u>0$,
\[
\Pr\!\left(\left\|\sum_i\mathbf X_i\right\|_2\geq u\right)
\leq2d_K\exp\!\left\{-\frac{u^2/2}{B\|\mathbf G\|_2/(\pi n)+Bu/(3\pi n)}\right\}.
\]
For $v_B=B\|\mathbf G\|_2/(\pi n)$ and $R_B=B/(\pi n)$, the choice $u=\sqrt{2v_Bt_\delta}+2R_Bt_\delta/3$ satisfies $u^2\geq2t_\delta(v_B+R_Bu/3)$. The failure probability is at most $2d_Ke^{-t_\delta}=\delta/q$. A union bound proves \eqref{eq:main_mask} simultaneously for all candidates.

\noindent\textbf{4. Control a changed design.}  On this event, $\mathbf A_j=\mathbf M\mathbf F_j/\sqrt{\pi n}$ has singular values between $\sqrt{1-\epsilon_j-\eta_j}$ and $\sqrt{1+\epsilon_j+\eta_j}$ whenever $\epsilon_j+\eta_j<1$. Let $\mathbf D_j=\mathbf M\mathbf E_j/\sqrt{\pi n}$. For every unit vector $\mathbf v\in\mathbb R^{d_K}$, the triangle and reverse triangle inequalities give
\[
\begin{split}
\|(\mathbf A_j+\mathbf D_j)\mathbf v\|_2
&\leq\sqrt{1+\epsilon_j+\eta_j}+\zeta_j,\\
\|(\mathbf A_j+\mathbf D_j)\mathbf v\|_2
&\geq\bigl(\sqrt{1-\epsilon_j-\eta_j}-\zeta_j\bigr)_+.
\end{split}
\]
Squaring these inequalities and optimizing over $\mathbf v$ proves \eqref{eq:main_drift}. This step is deterministic and thus permits $\mathbf E_j$ to depend on the sampled mask. For quadratic loss with Hessian $\widetilde{\mathbf Q}_j$, gradient descent satisfies $\mathbf e_{t+1}=(\mathbf I-\widetilde{\mathbf Q}_j/b_j)\mathbf e_t$. Its spectral norm is at most $1-a_j/b_j$ when $a_j>0$, proving the stated contraction. 
\end{proof}

\subsection{Proof of the risk decomposition}
\label{proof:risk}
\begin{proof}[Proof of Proposition~\ref{prop:risk}]
Let $d_K=K+1$ and use the notation in the statement.

\noindent\textbf{1. Solve the normal equations.}  Differentiating the quadratic objective gives
\[
\widehat{\boldsymbol\theta}=\mathbf A^{-1}\mathbf F^\top\mathbf M\mathbf y/m_{\mathcal T},\qquad
\widehat{\boldsymbol\theta}-\boldsymbol\theta^\star
=-\mathbf A^{-1}2\mathbf W\boldsymbol\theta^\star
+\mathbf A^{-1}\mathbf F^\top\mathbf M\boldsymbol\xi/m_{\mathcal T}.
\]
\noindent\textbf{2. Compute the noise covariance.}  Since $\mathbf M^2=\mathbf M$, the second term has mean zero and covariance
\[
\frac{\sigma^2}{m_{\mathcal T}^2}\mathbf A^{-1}\mathbf F^\top\mathbf M\mathbf F\mathbf A^{-1}
=\frac{\sigma^2}{m_{\mathcal T}}\mathbf A^{-1}\mathbf Q\mathbf A^{-1}.
\]
\noindent\textbf{3. Evaluate prediction risk.}  The identity $\|\mathbf F\mathbf v\|_2^2/n=\mathbf v^\top\mathbf G\mathbf v$ and the vanishing cross term give \eqref{eq:main_risk}. If $\mathbf G=\mathbf Q=\mathbf I_{d_K}$, then $\mathbf A=\operatorname{diag}(1+u_0,\ldots,1+u_K)$, so
\[
\mathcal R(\mathbf W)=\sum_{k=0}^K\frac{u_k^2(\theta_k^\star)^2+v}{(1+u_k)^2},\qquad
\mathcal R(\mathbf0)=d_Kv.
\]
Subtracting proves \eqref{eq:main_shrinkage}.

\noindent\textbf{4. Transform the complete objective.}  For the change-of-basis result, let $\mathbf F_b=\mathbf F\mathbf T$, $\mathbf W_b=\mathbf T^\top\mathbf W\mathbf T$, and $\mathbf A_b=\mathbf T^\top\mathbf A\mathbf T$. Invertibility of $\mathbf T$ gives
\[
\widehat{\boldsymbol\beta}
=\mathbf A_b^{-1}\mathbf T^\top\mathbf F^\top\mathbf M\mathbf y/m_{\mathcal T}
=\mathbf T^{-1}\widehat{\boldsymbol\theta},\qquad
\mathbf F_b\widehat{\boldsymbol\beta}=\mathbf F\widehat{\boldsymbol\theta}.
\]
Positive definiteness of $\mathbf A$ ensures uniqueness. Without uniqueness, equality of objective values alone does not specify a common fitted signal away from the labeled nodes. 
\end{proof}

\subsection{Proof of the representation decomposition}
\label{proof:representation}
\begin{proof}[Proof of Proposition~\ref{prop:representation}]
\noindent\textbf{1. Separate orthogonal components.}  Since $\widehat{\mathbf f}$ and $\boldsymbol\Pi_K\mathbf f$ belong to $\mathcal U_K$, their difference is orthogonal to $\mathbf f-\boldsymbol\Pi_K\mathbf f$. Expanding
\[
\widehat{\mathbf f}-\mathbf f
=(\widehat{\mathbf f}-\boldsymbol\Pi_K\mathbf f)
-(\mathbf f-\boldsymbol\Pi_K\mathbf f)
\]
and using this orthogonality proves \eqref{eq:main_representation}. \noindent\textbf{2. Change coordinates.}  An invertible basis transformation preserves the design's column space and its orthogonal projector. \noindent\textbf{3. Increase the degree.}  Finally, $\mathcal U_K\subseteq\mathcal U_{K+1}$ because the degree-$K$ polynomial class is nested in the degree-$(K+1)$ class for the same input. Minimizing squared distance over the larger space cannot increase it, so $B_{K+1}\leq B_K$. All three statements remain valid for a rank-deficient design. 
\end{proof}

\subsection{Proof of the perturbation bound}
\label{proof:stability}
\begin{proof}[Proof of Theorem~\ref{thm:stability}]
\noindent\textbf{1. Bound the scalar derivative.}  By the Hermite derivative identity and Cauchy--Schwarz, for $z=(\lambda-\mu)/s$,
\[
|g'(\lambda)|
=\left|\sum_{k=1}^K\frac{\sqrt{k}\theta_k}{s}h_{k-1}(z)\right|
\leq\sqrt{\mathcal D_1(g)}\left(\sum_{k=0}^{K-1}h_k(z)^2\right)^{1/2}
\leq\rho_K(r)\sqrt{\mathcal D_1(g)}.
\]
\noindent\textbf{2. Express both operators in mixed eigenbases.}  Let $\widetilde{\mathbf L}=\mathbf V\operatorname{diag}(\widetilde\lambda_i)\mathbf V^\top$ and $\mathbf L=\mathbf U\operatorname{diag}(\lambda_j)\mathbf U^\top$. The entries of the mixed eigenbasis representations satisfy
\[
\begin{split}
[\mathbf V^\top(g(\widetilde{\mathbf L})-g(\mathbf L))\mathbf U]_{ij}
&=(g(\widetilde\lambda_i)-g(\lambda_j))[\mathbf V^\top\mathbf U]_{ij},\\
[\mathbf V^\top(\widetilde{\mathbf L}-\mathbf L)\mathbf U]_{ij}
&=(\widetilde\lambda_i-\lambda_j)[\mathbf V^\top\mathbf U]_{ij}.
\end{split}
\]
\noindent\textbf{3. Sum the entrywise bounds.}  The scalar mean value theorem, followed by summing squared entries, gives
\[
\|g(\widetilde{\mathbf L})-g(\mathbf L)\|_F
\leq\rho_K(r)\sqrt{\mathcal D_1(g)}\,\|\widetilde{\mathbf L}-\mathbf L\|_F.
\]
\noindent\textbf{4. Add feature perturbation.}  Finally, decompose the signal difference as
\[
(g(\widetilde{\mathbf L})-g(\mathbf L))\mathbf h
+g(\widetilde{\mathbf L})(\widetilde{\mathbf h}-\mathbf h).
\]
Applying $\|\mathbf A\mathbf h\|_2\leq\|\mathbf A\|_F\|\mathbf h\|_2$ for the first term and $\|g(\widetilde{\mathbf L})\|_2\leq M_g$ for the second proves \eqref{eq:main_stability}. 
\end{proof}

\subsection{Proof of the degree--scale constraint}
\label{proof:support}
\begin{proof}[Proof of Equation~\eqref{eq:main_support}]
\noindent\textbf{1. Use the Gram bounds.}  For $p_{\mathbf a}=\sum_{k=0}^Ka_kh_k$, the assumption gives
\[
(1-\epsilon)\|\mathbf a\|_2^2\leq\int p_{\mathbf a}^2\,d\nu\leq(1+\epsilon)\|\mathbf a\|_2^2.
\] \noindent\textbf{2. Apply the recurrence and support restriction.}  For $K\geq2$, the recurrence gives $zh_{K-1}=\sqrt K h_K+\sqrt{K-1}h_{K-2}$. For $K=1$, it gives $zh_0=h_1$. Therefore, in both cases,
\[
(1-\epsilon)(2K-1)
\leq\int z^2h_{K-1}(z)^2\,d\nu(z)
\leq r^2\int h_{K-1}(z)^2\,d\nu(z)
\leq r^2(1+\epsilon).
\]
\noindent\textbf{3. Solve for the support radius.}  Rearrangement proves \eqref{eq:main_support}. The condition is necessary but does not guarantee the existence of an aligned measure on a given graph. 
\end{proof}

\subsection{Proof of the product-graph moments}
\label{proof:product_moments}
\begin{proof}[Proof of Equation~\eqref{eq:product_moments}]
\noindent\textbf{1. Diagonalize the coordinate flips.} 
For a $q$-dimensional weighted Boolean product, let $\mathbf P_\ell$ flip coordinate $\ell$. These commuting symmetric matrices satisfy $\mathbf P_\ell^2=\mathbf I$. Walsh characters have eigenvalues $\varepsilon_\ell\in\{-1,1\}$, and
\[
 \mathbf A=\sum_{\ell=1}^qw_\ell\mathbf P_\ell,\qquad
 \mathbf D=\Bigl(\sum_\ell w_\ell\Bigr)\mathbf I,\qquad
 \lambda_{\boldsymbol\varepsilon}=1-\frac{\sum_\ell w_\ell\varepsilon_\ell}{\sum_\ell w_\ell}.
\]
\noindent\textbf{2. Standardize the uniform eigenvalue measure.} 
All $2^q$ sign vectors occur once, so uniform spectral sampling makes the signs independent and equiprobable. With $a_\ell=w_\ell/(\sum_jw_j^2)^{1/2}$, symmetry gives
\[
 Z\overset d=\sum_\ell a_\ell\varepsilon_\ell,\qquad
 \mathbb E[Z]=0,\qquad \mathbb E[Z^2]=\sum_\ell a_\ell^2=1.
\]
\noindent\textbf{3. Expand the fourth power.} 
Products containing an odd power of any sign have zero expectation. Therefore
\[
 \mathbb E[Z^4]=\sum_\ell a_\ell^4+6\sum_{\ell<j}a_\ell^2a_j^2
 =3\Bigl(\sum_\ell a_\ell^2\Bigr)^2-2\sum_\ell a_\ell^4
 =3-2\frac{\sum_\ell w_\ell^4}{(\sum_\ell w_\ell^2)^2}.
\]
This proves \eqref{eq:product_moments}. The statement concerns the unweighted spectral measure. A realized signal supplies different weights in \eqref{eq:measure}, so its moments are measured separately. \end{proof}

\FloatBarrier
\small
\section{Experimental Details}
\label{app:protocol}

The protocols distinguish fixed-design diagnostics, learned predictors and native architectures. Synthetic follow-up data are generated after the relevant settings are fixed. Real-data studies use the declared reused datasets and splits.

We select hyperparameters and checkpoints using validation labels. Synthetic prediction summaries average optimizer or reference-initialization repetitions within a graph before computing graph means and standard deviations. Conditions on the same graph are paired and are not independent replications. Real-data means and standard deviations use three optimizer seeds on one fixed split per dataset. They describe optimization variability on those graphs, not uncertainty across graph populations. No statistical significance or equivalence claim is based on these standard deviations.
\subsection{Graph, signal, and target construction}
\label{app:generators}

The stochastic-block model has three nearly equal blocks and edge probabilities $p_{\rm in}=0.06$ and $p_{\rm out}=0.01$. For $n=400$, block sizes differ by at most one. The geometric model places nodes uniformly in $[0,1]^2$, joins each node to its ten nearest neighbors, and symmetrizes the edges. Both use unit edge weights and no added self-loops. The normalized Laplacian and isolated-node convention are those of Section~\ref{sec:problem}. Prediction targets are generated from graph signals rather than assigned from block membership.

Let $\ell_1<\cdots<\ell_J$ be distinct eigenvalues after grouping values within $10^{-9}$, and let $\mathbf U_j$ be an orthonormal basis of the corresponding eigenspace. Midpoint boundaries, including 0 and 2 at the ends, give widths $\Delta_j$. For a nonnegative density shape $p$, define $w_j=p(\ell_j)\Delta_j/\sum_a p(\ell_a)\Delta_a$ and generate
\[
\mathbf r=\sqrt n\sum_{j=1}^J\sqrt{w_j}\,\mathbf U_j\mathbf v_j,
\qquad \|\mathbf v_j\|_2=1,
\]
where each $\mathbf v_j$ is sampled uniformly from its unit sphere. This construction makes the total energy of an eigenspace invariant to its chosen basis. Changing the density changes the weights while retaining paired directions. Gaussian-shaped energy uses $p(\lambda)=\exp[-(\lambda-1)^2/(2\tau^2)]$ with $\tau=0.2$ or $0.4$. Uniform energy uses $p=1$. The classification factorial also uses $p(\lambda)=\exp[-(\lambda-0.55)^2/(2(0.12)^2)]+\exp[-(\lambda-1.45)^2/(2(0.12)^2)]$. These densities prescribe signal energy on the spectrum and do not describe the distribution of graph eigenvalues.

For scalar regression, the target is $h(\mathbf L)\mathbf r$ rescaled to unit all-node RMS, with $h(\lambda)=e^{-2\lambda}$ or $\mathbb I\{\lambda\leq1\}$. Repeated eigenspaces use their shared representative eigenvalue for the cutoff in the fixed-feature studies. Gaussian observation noise has standard deviation $0.5$ and is paired across controlled conditions. The input-mismatch construction orthogonalizes an independent uniform-energy signal against $\mathbf r$, rescales it to norm $\sqrt n$, and forms $\mathbf z=(\mathbf r+\rho\mathbf u)/\sqrt{1+\rho^2}$. Calibration remains based on $\mathbf r$, whereas filtering uses $\mathbf z$. The target remains fixed, so input mismatch changes representation as well as calibration agreement.

The fixed-feature studies reserve 80 validation and 200 test nodes, leaving a 120-node training pool. The 30-label condition uses a nested subset of that pool, leaving the other 90 nodes unused for training. Training sees noisy training labels, and selection sees noisy validation labels. Noiseless targets are used only by the separate prediction evaluator and, subsequently, the empirical oracle diagnostic. The coordinate study's test targets remain unused.

\subsection{Calibration and equivalent-coordinate optimization}
\label{app:coordinate_protocol}

Calibration uses the reference moment center and scale with $s_{\min}=0.05$, offsets $\{-0.5,0,0.5\}$, scale multipliers $\{0.5,0.75,1,1.5,2\}$, and the fixed pair $(1,1)$. Candidates satisfy radius at most six. Duplicates are removed, and ties follow lexicographic coordinate order. Reference whitening uses a full-rank singular-value decomposition of the reference Chebyshev design and remains fixed after input mismatch. Fixed Jacobi coordinates use parameters $(1/2,1/2)$. A relative singular-value tolerance of $10^{-10}$ defines numerical rank. Whitening requires full rank, without truncation.

Write the filter in Chebyshev coordinates as $g(\lambda)=\sum_{k=0}^K c_kT_k(\lambda-1)$, with coefficient vector $\mathbf c$. Let $\mathbf V$ evaluate these polynomials at the $K+1$ Chebyshev nodes, so $V_{jk}=T_k(\cos[\pi(j+1/2)/(K+1)])$ for $j,k=0,\ldots,K$. For the training design $\mathbf F_{\mathcal T}$ and noisy training labels $\mathbf y_{\mathcal T}$, the common coordinate-study objective is
\[
J(\mathbf c)=\frac{\|\mathbf F_{\mathcal T}\mathbf c-\mathbf y_{\mathcal T}\|_2^2}{2m}
+\frac{10^{-4}}{2(K+1)}\|\mathbf V\mathbf c\|_2^2.
\]
Every basis transforms the design, complete penalty, and identity-filter initialization. An augmented, column-equilibrated singular-value decomposition computes the exact minimizer with scaled stationarity tolerance $10^{-9}$. Gradient descent uses the exact full Hessian's largest eigenvalue. Relative gap divides the quadratic excess objective by the larger of its initial value and $10^{-12}$. The tolerance is $10^{-6}$, the cap is 2,000 updates, and finite trajectories that do not reach the tolerance enter the summaries with a censored count of 2,001. Exact cross-coordinate predictions must agree to relative tolerance $10^{-9}$.

The study uses eight graphs per family and crosses three energy distributions, two degrees, two label counts, two mismatch levels, two targets and four bases, giving 3,072 trajectories. Each trajectory has a matched exact solution. All numerical-rank and coordinate-equivalence checks pass. A Gaussian width/degree pair is evaluated against five criteria specified before the study: a stable Hermite median condition number of at most 30, a median paired Chebyshev-to-Hermite condition ratio of at least two, Hermite optimization success of at least $90\%$ with median paired $(1+t_H)/(1+t_C)$ at most $0.75$, a median paired mismatch-to-stable Hermite condition ratio of at most five, and exact noisy-validation MSE within $5\%$ of the better degree at that width. All criteria must hold in both graph families. Conditioning summaries pool graphs and label counts. Optimization and validation summaries also pool the two targets. Validation uses one canonical exact solution, so equivalent bases do not count as independent observations.

\subsection{Neural classification and native architectures}
\label{app:classification_protocol}

The synthetic classification factorial uses $n=600$, five graphs per family, 60 training nodes, 120 validation nodes, and 420 test nodes. Three independent graph signals form class-specific teacher logits after filtering, centering, and RMS normalization. Labels are sampled from their rowwise softmax. Inputs contain the three signals and 13 independent Gaussian nuisance columns, with transductive column standardization. The two graph families, three densities, and two targets yield 12 paired conditions, not 60 independent graph populations.

A width-64, two-layer ReLU reference predictor is trained using learning rate $0.003$, dropout $0.3$, explicit squared-parameter penalty $10^{-4}$, and a 300-epoch cap with patience 100. The frozen and jointly trained modes share this fitted predictor and its identity filter at initialization. Final reference-initialization repetitions are averaged within graph. Filter training has no dropout, learning rate $0.003$, a 1,000-epoch cap, and patience 100. Hermite and Chebyshev-degree controls each tune 18 settings crossing $K\in\{2,4\}$, coefficient ridge in $\{0,10^{-4},0.01\}$, and roughness strength in $\{0,0.001,0.1\}$. Chebyshev ridge tunes 18 settings with the same degrees and ridge in $\{0,10^{-6},10^{-5},10^{-4},0.001,0.01,0.03,0.1,0.3\}$. Mean graph validation accuracy, then cross-entropy, selects each method/mode/condition, with checkpoint ties resolved in favor of the earlier epoch. Each final setting has three reference-initialization repetitions. This warm-start factorial is distinct from the independent initialization used for practical node classification.

For the citation benchmarks \citep{kipf2017}, binary edges are symmetrized and deduplicated, and raw self-loops are removed before applying the method's operator. Features are row normalized. The CiteSeer isolated-node extension is retained, and all-zero unavailable label rows are excluded from the supervised split. One fixed random split per dataset uses 20 training nodes per class, 500 validation nodes, and 1,000 test nodes. All nodes and features remain visible transductively. The graphs were also used in earlier development, so this evaluation is not independent confirmation.

The practical \HermitePoly{} variant adds input, hidden, and pre-filter-logit dropout. It uses classwise filters, coefficient Adam weight decay, and a class-averaged derivative term $\lambda_1(C s^2)^{-1}\sum_{c,k}k\theta_{kc}^2$. Thus its numerical derivative strength corresponds to $\lambda_1/C$ in the class-sum convention of \eqref{eq:main_objective}. Coupled Adam decay corresponds to half that decay coefficient in an explicit squared-parameter penalty. A separate width-64 reference MLP uses learning rate $0.01$, decay $0.0005$, input/hidden dropout $0.5$, and a 300-epoch cap with patience 60. Final predictors are independently initialized, and reference/calibration costs are separate from final-model fitting.

Each of the six methods in Table~\ref{tab:real} receives 16 fixed configurations: eight initial configurations and eight additional settings that include predictor weight decay $0$, $10^{-5}$, $10^{-4}$, or $0.001$ and broader learning-rate/dropout choices. Polynomial degree choices include 2, 4, and 8, with degree 10 retained for applicable earlier native defaults. The finite configuration list contains selected combinations from these ranges. The source \OptBasisGNN{} PubMed degree-12 preset is outside this list. Two tuning optimizer seeds select by mean validation accuracy and then mean cross-entropy. Reporting uses three independent final seeds. The common training cap is 600, the minimum training duration is 200 epochs, patience is 150, and validation is checked every five epochs. The selected checkpoint may precede the minimum training duration. Selected configurations are fixed before test evaluation.

The native models use the authors' architectures, initializers and applicable safeguards. In particular, \JacobiConv{} keeps its linear predictor and learned coefficient decomposition, and \OptBasisGNN{} keeps its normalization and small evaluation perturbation with reproducible evaluation randomness. The fixed-graph sparse aggregation backend was checked against the original values and gradients. Classification experiments use CPU float32, so their timings are not directly comparable with the authors' GPU measurements. The public source revisions are \href{https://github.com/ivam-he/ChebNetII/tree/ded6c18cbe9673234071031767d17826ad632aca}{\ChebNetII}, \href{https://github.com/GraphPKU/JacobiConv/tree/5e9e671fac63e680e0681fa9b4d8074960d2d65e}{\JacobiConv}, and \href{https://github.com/yuziGuo/FarOptBasis/tree/1e3fdac8ea03b8c98110f740cd79afea9fd4831b}{\FavardGNN/\OptBasisGNN}. Exact configuration lists, selected settings, and source identities are documented in the experimental records.

\subsection{Learned-feature regression}
\label{app:regression_protocol}

The learned regression task uses the same two graph families with $n=400$, Gaussian width-$0.4$ or uniform latent energy, and smooth or cutoff targets. Its observable features are $\operatorname{asinh}(\mathbf r)$ plus seven independent Gaussian nuisance columns, standardized over all nodes. The learner is not supplied with $\mathbf r$ or the inverse transformation. Splits contain 60 training, 80 validation, and 260 test nodes. The reference is a separately fitted width-32 ReLU MLP. Final predictors are independently initialized. The four matched bases share the degree-bounded class, identity filter, predictor, and complete objective: training MSE/2, predictor L2, and a $10^{-4}/2$ mean squared response penalty on Chebyshev nodes, with zero derivative weight in the native-comparison study.

Each method starts with 36 joint configurations. Validation on two graphs per family with two optimizer seeds retains the best eight, which are then evaluated on six additional tuning graphs. Mean noisy-validation MSE over all eight tuning graphs selects the final configuration among those eight configurations. This procedure can discard a configuration that would have performed well on the additional graphs. Evaluation uses eight new graphs per family and two new optimizer seeds per graph. Graphs are shared across the four density/target conditions in each family. We report noiseless test MSE and graph SD after averaging optimizer seeds within graph.

The matched and native lists use degrees 2/4 and method-specific learning rates and penalties. Native predictor L2 strengths are $0$, $0.001$ and $0.01$, and predictor learning rates are $0.003$, $0.01$ and $0.03$. Filter learning rates range from $0.001$ to $0.3$, extending to 1 for $\OptBasisGNN{}$. Native filter decay choices are $0$, $10^{-6}$, and $10^{-4}$. Jacobi shape and coefficient-decomposition settings, and Favard recurrence rates, retain their architecture-specific roles. The configuration lists have balanced search allocations but preserve method-specific priors and capacities. Native \JacobiConv{} retains a linear predictor, while other hidden predictors use width 32. Fits use CPU float64, no dropout or clipping, a 1,200-epoch cap, minimum duration 400, patience 300, and validation at epoch one and every five epochs thereafter. The stopping rule bounds training duration without guaranteeing convergence.

\subsection{Derivative-prior selection and error decomposition}
\label{app:prior_protocol}

For a Chebyshev response $g(\lambda)=\sum_{k=0}^Kc_kT_k(\lambda-1)$, define the raw quadratic priors
\[
\begin{aligned}
R_F(g)&=\int_{\mathbb R}g'(\lambda)^2p_{\mu,s}(\lambda)\,d\lambda,
& R_B(g)&=\int_0^2g'(\lambda)^2p_{\mu,s}(\lambda)\,d\lambda,\\
R_U(g)&=\tfrac12\int_0^2g'(\lambda)^2\,d\lambda,
& R_C(g)&=\sum_{k=1}^K k^2c_k^2,
\end{aligned}
\]
where $p_{\mu,s}$ is the calibrated Gaussian density. The bounded integral is not initially divided by its retained probability. For each matrix $\mathbf Q$ representing one of these forms, normalize by $a_Q=\operatorname{tr}(\mathbf V^{-\top}\mathbf Q\mathbf V^{-1})/K$ and put $\widetilde{\mathbf Q}=\mathbf Q/a_Q$. Dividing the bounded integral by its retained probability before this normalization would give the same normalized matrix. The fitted objective is
\[
\frac{\|\mathbf F_{\mathcal T}\mathbf c-\mathbf y_{\mathcal T}\|_2^2}{2m}
+\frac{10^{-4}\|\mathbf V\mathbf c\|_2^2}{2(K+1)}
+\frac{\eta}{2}\mathbf c^\top\widetilde{\mathbf Q}\mathbf c.
\]
Zero roughness uses $\eta=0$. Common trace normalization controls penalty scale in response coordinates, although effective degrees of freedom can still differ. The constant nullspace is retained exactly. Each complete penalty is transformed into equivalent Hermite coordinates, rather than independently normalized after the coordinate change.

For each of eight tuning graphs per family, the study crosses three densities, two targets, two training counts, two mismatch levels, and two degrees. Each context fits the shared zero-roughness objective and four priors at $\eta\in\{10^{-5},10^{-4},10^{-3},10^{-2},0.1,1\}$. Within each of the 48 family/density/target/label-count/mismatch conditions, noisy-validation MSE averaged over tuning graphs selects a strength at each degree and then the degree. Exact ties prefer smaller strength and smaller degree. The selections are fixed before generating the 12 new evaluation graphs per family. The 48 conditions use shared graphs, so they are not 48 independent experiments.

The primary block fixes Gaussian width $0.4$ and the smooth response and averages the four training-count/mismatch conditions within each graph. The stated practical target is a gain of at least $0.005$ absolute MSE and $5\%$ relative MSE against each of four comparators in each family. Forty remaining conditions yield 160 comparator-condition checks for mean harm above $0.005$. These are descriptive practical criteria, not significance tests. All required fits completed without failure. Reported prediction errors use only the independent evaluation targets. No noiseless tuning targets or test targets from the coordinate study enter these errors.

The subsequent oracle diagnostic projects the noiseless target on the degree-two or degree-four evaluation design for all saved predictions. SVD and QR projections agree, as do equivalent-basis projections. The calculation contains 576 target-specific oracle projections: two families, 12 graphs per family, three densities, two mismatch levels, two targets, and two degrees. The complete output has 5,760 fixed-degree and 2,880 selected-degree procedure outcomes. The diagnostic uses the previously evaluated test nodes and is conducted after fitting and evaluation. It leaves the fitted models and evaluated conditions unchanged. The representation floor is computed on a finite sample using the known target. The within-space term includes bias, noise sensitivity, sampling and selection.

\subsection{Expanded common-model and native comparisons}
\label{app:later_protocol}

The common-architecture comparison uses Hermite, Chebyshev, unrestricted Bernstein, Jacobi with shape $(0.5,0.5)$, Legendre and signed monomial bases. A separate empirical QR control uses the same fixed input responses when included. Native models are \ChebNet{}, \ChebNetII{}, \BernNet{}, \JacobiConv{}, \GPRGNN{}, \FavardGNN{}, \OptBasisGNN{}, \APPNP{}, \texttt{GCN} and \texttt{MLP}. Native constraints and architecture are preserved, so these models need not match the common model's parameter count. The standard deviations in this comparison are descriptive. These development evaluations do not constitute independent confirmatory tests. The highest comparator means summarize the evaluated procedures without using test labels for model selection.

For the product graphs, nodes are binary vectors of length nine and an edge flipping coordinate $j$ has weight $w_j$. Balanced weights are iid uniform on $[0.75,1.25]$. The dominant condition replaces $w_1$ by four. With $s_w=\|\mathbf w\|_2/\sum_jw_j$, write $z=(\lambda-1)/s_w$. The diffusion task draws a latent iid standard Gaussian signal $\mathbf r$, observes its componentwise inverse-hyperbolic-sine transform plus seven independent Gaussian nuisance columns, and standardizes each observed column. The clean signal is $\exp[-0.6(\mathbf L-\mathbf I_n)/s_w]\mathbf r$, normalized to unit nodewise RMS. Labels threshold this signal after independent Gaussian noise of standard deviation 0.25. Random masks contain 80 training, 100 validation and 332 evaluation nodes.

The product-graph tuning menu fully crosses degree $\{2,4\}$, learning rate $\{0.01,0.03\}$, affine versus width-32 nonlinear predictor, dropout $\{0,0.5\}$, and fixed/raw versus input-moment/RMS coordinates. QR uses its full response transform. Native menus use degrees $\{2,4\}$, rates $\{0.003,0.01,0.03,0.1\}$, the same two dropout choices and decay $\{0.0005,0.005\}$. The \texttt{GCN}/\texttt{MLP} configurations vary hidden width between 16 and 32 instead of propagation degree. Each of 32 configurations trains for 100 updates on four tuning graph/optimizer pairs. The four highest-ranked configurations continue to 500 updates. Selection uses average best validation accuracy, then cross-entropy and configuration order. Eight independent graph/optimizer pairs receive 1,000-update refits. One optimizer is used per graph in this study. Raw nonlinear/no-dropout and fixed linear/dropout common references are evaluated as separate reference procedures.

The separate CiteSeer search uses the public partition, 16 configurations per method and two tuning seeds. Configurations are reduced to four, two and one at budgets 100, 300 and 1,000, respectively, before three new-seed 1,000-update refits. Common degree-four models use optional fixed/moment coordinates, raw/RMS normalization and derivative regularization. The selected Hermite configuration uses moment/RMS coordinates, rate 0.01 and zero derivative penalty. The dataset and partition have already been used in development. Early pruning can discard configurations that improve with longer training. Equal update allocations also leave differences in computation cost and search coverage across native architectures.

WebKB features are row-normalized. Edges are symmetrized, made binary and stripped of self edges. The supplied masks contain 87/59/37 training/validation/test nodes for Cornell and Texas and 120/80/51 for Wisconsin, approximately 48/32/20 percent. These counts follow the supplied masks. Common menus cross predictor type, dropout and learning rate at degree four with input-moment/RMS normalization. QR uses a full transform. For every method and split, eight configurations train to 100 updates, two continue to 500, and the selected configuration receives two new-seed 1,000-update refits. Selection uses validation labels only. Accuracy-first and CE-first checkpoints are evaluated on the same selected configurations. Each of ten split means first averages the two optimizer seeds. Native architecture, capacity and computation remain method specific.

\subsection{Channel-filter, rounding and sparse-projection protocols}
\label{app:constrained_protocol}

The channel-filter task uses eight new balanced product graphs, each with eight iid Gaussian feature columns. Four informative columns are independently filtered by $\cos(0.4z)$, $\sin(0.4z)$, $\cos(1.4z)$ and $\sin(1.4z)$. Each response is divided by its spectral RMS and then by two, giving unit expected total signal variance over the feature draws. Noise of standard deviation 0.25 is added before thresholding. We use 80/100/332 nodes and two optimizer seeds per graph. This differs from the diffusion task above and does not use Hermite teacher coefficients.

Full and factorized models cache the same degree-four input responses, use input-moment/RMS normalization or QR, identity-filter initial predictions, zero parameter decay, Adam rate 0.03 and dropout 0.5. A node/feature dropout mask is shared across orders, preserving its objective under an invertible order-basis change. Fits run 1,000 updates with validation every five updates. All ten native reference models use the settings selected on the diffusion task. The additional 82-parameter single-layer \texttt{ChebConv} has degree four, input dropout 0.5, rate 0.03 and decay 0.0005 excluding bias. Its initialization and dropout placement differ from those of the cached common models. Means first average optimizer seeds within graph.

The precision and sparsity diagnostics use separate sets of eight independent graphs and both weight conditions. Their spectra are available exactly from the product construction. Each of six polynomial families is evaluated in moment argument $z$ and fixed argument $\lambda-1$, with columns normalized by spectral RMS. QR is constructed from bounded Chebyshev columns. These thirteen representations include the usual bounded spectral interval for fixed Chebyshev and Bernstein. Fixed Hermite is an ablation. All unrestricted degree-$K$ projected responses agree numerically before constraints are applied.

For coefficient precision, we use three four-channel target groups: the wave responses defined above, smooth responses $e^{-0.5\lambda}$, $e^{-2\lambda}$, $1-e^{-0.5\lambda}$ and $1-e^{-2\lambda}$, and indicators of $z$ in $[-1.5,-0.5)$, $[-0.5,0.5)$, $[0.5,1.5)$ and $|z|\geq1.5$. Each column has spectral RMS one before division by two. We project each target group onto degree four or eight and recover coefficients in every basis. For bit depth $b\in\{4,8\}$, a coefficient group with maximum magnitude $M>0$ uses scale $a=M/(2^{b-1}-1)$ and rounded coefficient $a\,\operatorname{round}(c/a)$, with ties rounded to even. Zero groups remain zero. The primary analysis uses one scale for the full tensor, and a sensitivity analysis uses one scale per input filter. Scales, responses and accumulation remain float64. The primary summary weights the three target groups equally within each graph.

Let $\boldsymbol\Phi$ be the normalized response matrix, $\mathbf C$ the unrounded projected coefficient matrix and $\boldsymbol\Delta$ the rounding error. Squared response error is exactly $\|\boldsymbol\Phi\boldsymbol\Delta\|_F^2=\operatorname{tr}(\boldsymbol\Delta^\top\boldsymbol\Phi^\top\boldsymbol\Phi\boldsymbol\Delta)$. Equal independent coefficient-noise variance would give the same expected error for all RMS-normalized designs with equal trace. Actual rounding also depends on coefficient range and error direction. Orthogonality alone is therefore insufficient to predict rounding error.

For sparse terms, degree is eight and all supports of size two, three or four among nine basis functions are enumerated. The same support is shared across four input channels, and coefficients are least-squares oracle projections of the known target. Wave frequencies are $(0.5,1)$, $(1.5,2)$ and $(2.5,3)$, with the middle pair and three terms primary. A Bernstein index selects a degree-eight basis function, not a degree equal to that index. Channel rotation preserves the shared-support approximation objective. The comparison measures approximation within constrained dictionaries using oracle support selection.

For either oracle diagnostic, if the target and fixed prediction applied to independent iid Gaussian channels are jointly Gaussian and independent Gaussian score noise has standard deviation 0.25, their correlation $\rho$ gives sign agreement $1/2+\arcsin(\rho)/\pi$. We compute $\rho$ from spectral inner products, including the noise variance. These values are calculated analytically, without sampling features or labels, training classifiers or measuring hardware inference. The quantity is conditional on a fixed filter and graph, not an upper bound on a transductive learner fitted to a realized feature matrix.

The CPU timing study uses the earlier diffusion-task inputs and degree-four affine/dropout procedures. It includes sparse operator construction, coordinate construction, model/optimizer setup, 100 updates, validation, saved-state restoration and inference. Imports, disk loading and configuration lookup are excluded. It crosses input versus expected-dropout calibration and response-side versus equivalent coefficient-side transforms. Coordinate construction has 31 timing repetitions per graph, and full fits have three. Medians are computed within graph and then across eight graphs per condition. The expected-dropout implementation uses the known product spectrum, so these timings do not establish generic-graph or GPU preprocessing costs.

\FloatBarrier
\section{Additional Results and Ablations}
\label{app:supplement}
Bold marks the best displayed mean in each comparable setting, including ties at the reported precision. Lower MSE, distortion and finite optimization counts are better, while higher accuracy is better. Bold values with $\dagger$ identify the principal HermNet evidence discussed in the main text, including its planned paired contrasts. These marks do not imply statistical significance. Configuration tables are not ranked, and captions specify the convention for contrasts and timing ratios.
\subsection{Downstream regression}
\label{app:learned_results}

\FloatBarrier
\subsubsection{Learned predictors: independent evaluation and earlier study}
\label{app:learned_five}
\label{app:learned_product}
\label{sec:learned_product}
Section~\ref{sec:main_learned} defines the common architecture and menus. The independent 80-draw follow-up uses seeds 95000--95079, with sample size and five-update endpoint fixed before generation. Eleven one-sided paired-$t$ bounds use Bonferroni familywise level $0.05$ within this follow-up. There is no sequential stopping. A 50,000-resample paired bootstrap uses the same tail probability as a descriptive check. Tables~\ref{tab:learned_five} and~\ref{tab:learned_five_bounds} report the means and contrasts. All adjusted lower bounds are positive. Nonzero curvature is selected in 38/80 Hermite fits, and every selected Hermite predictor changes from initialization.

The earlier study screens sine, cosine and exponential teachers $\sin(0.75z)$, $\cos(0.75z)$ and $e^{0.4z}$, training counts 32/80 and noise SDs 0.3/0.9 on four reused draws (93000--93003). A plain candidate requires a 10\% lead over its best plain rival. An enhanced candidate requires 10\% gains over both its best enhanced rival and plain Hermite. Joint selection maximizes the smallest gain, with lower-budget and lexical ties. This fixes the sine/32-label/SD-0.3 task and the primary 20-update budget before 40 new draws (94000--94039). The selected exploratory 20-update gains are 12.0\% against the best plain rival, 16.7\% against the best enhanced rival and 11.5\% over plain Hermite. Its five-update result is secondary and motivates the independent follow-up. Neither arm resolves superiority over every rival at 20 updates, although the incremental Hermite benefit is resolved. Figure~\ref{fig:learned_budgets} compares budgets in this earlier study.

For both studies, predictor weight SDs are $1/2$ and $1/4$, biases are zero, and the shared initial filter is zero. Each update computes both gradients before either step. Adam moments are $(0.9,0.999)$ with $\epsilon=10^{-8}$. The first predictor update is zero. Random-permutation positions 97--160 are validation nodes and 161--256 test nodes, with the first 32 or 80 used for training. Validation checks every update and breaks ties by earlier checkpoint, then setting order. The enhanced menu includes all plain trajectories. In aligned Hermite coordinates the curvature matrix is $\operatorname{diag}(64k(k-1))$. Predictor rescaling can change effective regularization strength. Saved curves, states, sparse-recurrence replay and independent penalty calculations reproduce the reported selections and scores.

\begin{table}[!htbp]
\centering\footnotesize\setlength{\tabcolsep}{4pt}
\caption{Jointly trained nonlinear predictor: clean test MSE, mean $\pm$ sample SD. Left: independent five-update follow-up on 80 paired draws. Right: the earlier primary 20-update study on 40 different draws. Columns belong to distinct evaluations. Bold compares bases within each column, and $\dagger$ marks the principal independent HermNet result.}
\label{tab:learned_five}\label{tab:learned_primary}
\begin{tabular}{@{}lcccc@{}}
\toprule
 & \multicolumn{2}{c}{Five updates, 80 draws} & \multicolumn{2}{c}{20 updates, 40 draws}\\
\cmidrule(lr){2-3}\cmidrule(l){4-5}
Basis & Plain & Curvature & Plain & Curvature\\
\midrule
Hermite & $\key{0.34671}\pm0.17850$ & $\key{0.33048}\pm0.17748$ & $\best{0.24596}\pm0.07416$ & $\best{0.22172}\pm0.05512$\\
Chebyshev & $0.47850\pm0.24754$ & $0.46594\pm0.25408$ & $0.31178\pm0.14181$ & $0.26458\pm0.09755$\\
Bernstein & $0.38266\pm0.21185$ & $0.37405\pm0.20223$ & $0.24794\pm0.06100$ & $0.24121\pm0.05693$\\
Jacobi & $0.45375\pm0.24043$ & $0.43656\pm0.23010$ & $0.29565\pm0.13030$ & $0.25936\pm0.09915$\\
Legendre & $0.51357\pm0.27557$ & $0.50709\pm0.27154$ & $0.29928\pm0.15392$ & $0.25935\pm0.12550$\\
Power & $0.50562\pm0.23089$ & $0.50049\pm0.23569$ & $0.31315\pm0.11532$ & $0.28551\pm0.09086$\\
\bottomrule
\end{tabular}
\end{table}

\begin{table}[!htbp]
\centering
\caption{All planned comparisons in the independent five-update experiment. Difference is comparator MSE minus Hermite MSE in the same arm. Own compares plain with enhanced Hermite. Positive values favor Hermite or the enhancement. Wins count strict paired improvements among 80 draws.}
\label{tab:learned_five_bounds}
\small
\begin{tabular}{@{}llrrrr@{}}
\toprule
Comparator & Arm & Difference & $t$ lower & Boot. lower & Wins\\
\midrule
Chebyshev & plain & 0.13179 & 0.08940 & 0.09179 & 70\\
Bernstein & plain & \key{0.03595} & \key{0.01313} & 0.01482 & 57\\
Jacobi & plain & 0.10704 & 0.07013 & 0.07293 & 69\\
Legendre & plain & 0.16686 & 0.11404 & 0.11825 & 73\\
Power & plain & 0.15891 & 0.11181 & 0.11513 & 77\\
Chebyshev & enhanced & 0.13546 & 0.09113 & 0.09414 & 70\\
Bernstein & enhanced & \key{0.04357} & \key{0.02297} & 0.02452 & 62\\
Jacobi & enhanced & 0.10608 & 0.07034 & 0.07306 & 70\\
Legendre & enhanced & 0.17662 & 0.12379 & 0.12847 & 75\\
Power & enhanced & 0.17001 & 0.12120 & 0.12512 & 77\\
Hermite & own & \key{0.01623} & \key{0.00608} & 0.00724 & 31\\
\bottomrule
\end{tabular}
\end{table}

\begin{figure}[!htbp]
\centering
\includegraphics[width=\linewidth]{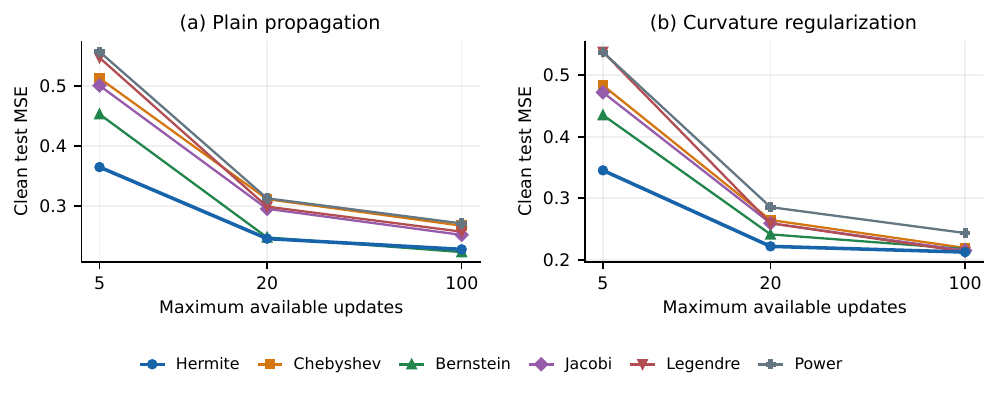}
\caption{All-basis budget sensitivity in the earlier learned-predictor study. Points are mean clean test MSE over the same 40 paired draws. Each point uses validation selection within its maximum update budget. The primary endpoint is 20 updates. Lines connect evaluated budgets, not individual training trajectories.}
\label{fig:learned_budgets}
\end{figure}

\FloatBarrier

\label{app:new_downstream}

\FloatBarrier
\subsubsection{Plain propagation and update budgets}
\label{sec:early_downstream}
\label{app:broad_scale_recovery}
The ten-dimensional weighted Boolean graph has 1,024 nodes, weights iid $U(0.75,1.25)$, a standard-normal scalar input and 80/100/844 training/validation/test nodes. The dominant condition sets $w_1=8$, changing scale, shape and task difficulty. Clean targets have unit RMS and observation noise SD 0.3. The original screen crosses two weight conditions, responses $\sin(0.75z)$, $\cos(0.75z)$, $e^{0.4z}$ and $(1+0.5z^2)^{-1}$ with $z=\sqrt{10}(\lambda-1)$, degrees 2/4/6 and four draws per cell (96 datasets). No clear Hermite advantage occurs at the primary 100-update screening budget. The degree-four sine task is selected adaptively, then evaluated on 40 new draws per weight condition, with primary budgets five and 20.

The original menu uses scales $1$ and $1/\sqrt{10}$, zero coefficients, and step fractions $1/4,1/2,1$ divided by the maximum training-Gram eigenvalue. Validation selects checkpoint, scale and fraction. QR orthogonalizes the training design and is an enhanced control. Exact unregularized predictions agree across coordinates to $1.09\times10^{-11}$. These are update-budget comparisons, not sparse-implementation timings.

The broad follow-up offers eight prescribed scales and three fractions (24 raw candidates per basis), followed by an optional analytical RMS normalization under each family's defining measure. It uses Gaussian, arcsine, Jacobi$(0.5,0.5)$ or uniform weights as appropriate, without graph moments or labels. Hermite's duplicate normalization is omitted. Validation checks every update through 20 and every five thereafter. After examination on the original draws, the protocol is fixed before 40 new draws. Table~\ref{tab:broad_scales} reports the five-update means. Figure~\ref{fig:broad_budgets} shows their budget dependence, and Table~\ref{tab:early_bounds} gives the original two-scale paired bounds.

\begin{table}[!htbp]
\centering
\caption{Independent evaluation with the broader scale search: clean test MSE, mean $\pm$ sample SD over 40 paired draws, at five updates. Every basis receives the same scales and step fractions. The second menu additionally offers analytical unit-RMS columns. Lower is better.}
\label{tab:broad_scales}
\small
\begin{tabular}{@{}lcc@{}}
\toprule
Basis & Raw scale menu & Normalization allowed\\
\midrule
Hermite & $\key{0.00511}\pm0.00370$ & $\key{0.00511}\pm0.00370$\\
Chebyshev & $0.08365\pm0.04840$ & $0.04949\pm0.02923$\\
Bernstein & $0.02165\pm0.01377$ & $0.02107\pm0.01430$\\
Jacobi & $0.03875\pm0.03010$ & $0.02964\pm0.02022$\\
Legendre & $0.15570\pm0.07266$ & $0.04350\pm0.01951$\\
Power & $0.16781\pm0.03678$ & $0.10122\pm0.02396$\\
\bottomrule
\end{tabular}
\end{table}

\begin{figure}[!htbp]
\centering
\includegraphics[width=\linewidth]{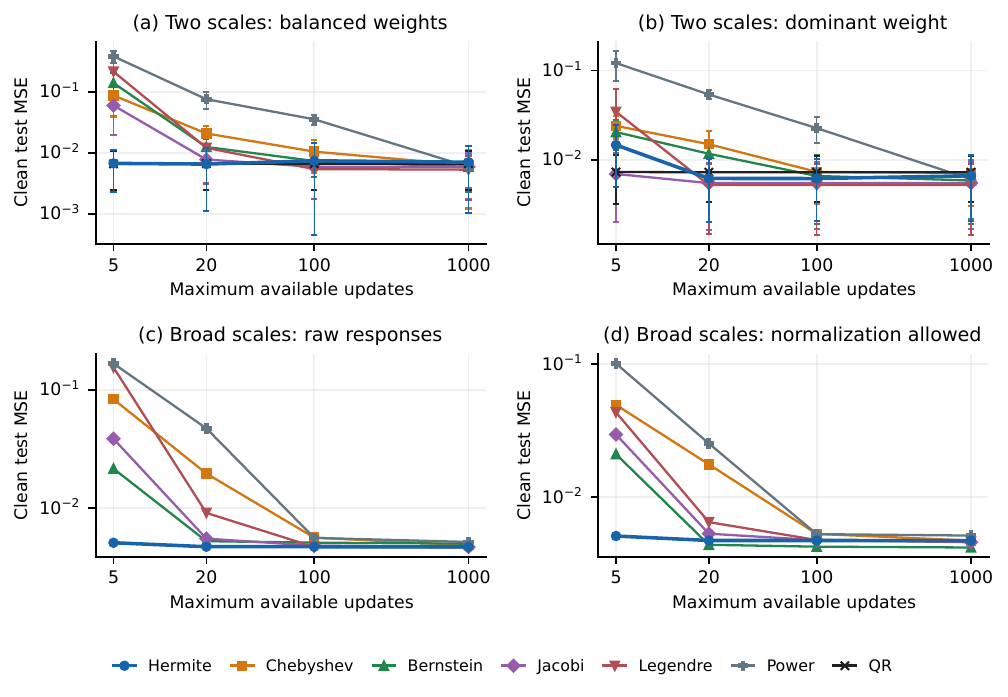}
\caption{Fixed-predictor budget sensitivity. Top: original two-scale evaluation, with mean $\pm$ graph SD across 40 independent draws per weight condition. QR is an enhanced control. Bottom: the separate broad-scale follow-up, showing means over 40 new paired draws, with and without an optional normalization menu. Both axes are logarithmic. Each point selects by validation within the available budget. Connecting lines are visual guides.}
\label{fig:early_balanced}
\label{fig:early_dominant}
\label{fig:broad_budgets}
\end{figure}

\begin{table}[!htbp]
\centering
\caption{Paired rival-minus-Hermite clean MSE on the 40 independent balanced graphs. Positive values favor Hermite. Lower bounds use a one-sided paired-$t$ approximation and Bonferroni adjustment across five rivals and two early budgets.}
\label{tab:early_bounds}
\small
\begin{tabular}{@{}lrrrrr@{}}
\toprule
Rival & Gain, 5 & Lower, 5 & Wins, 5 & Gain, 20 & Lower, 20\\
\midrule
Chebyshev & 0.08038 & 0.06042 & 40/40 & 0.01421 & 0.01156\\
Bernstein & 0.13258 & 0.10463 & 40/40 & 0.00592 & 0.00331\\
Jacobi & \key{0.05296} & \key{0.03622} & 40/40 & 0.00123 & -0.00069\\
Legendre & 0.20727 & 0.17364 & 40/40 & 0.00552 & 0.00160\\
Power & 0.37391 & 0.33782 & 40/40 & 0.06894 & 0.05944\\
\bottomrule
\end{tabular}
\end{table}

\FloatBarrier
\subsubsection{Derivative, curvature and odd-hop controls}
\label{sec:enhancement_value}
\label{app:enhancement_checks}
The derivative follow-up adds half the weighted norm $\mathbb E_{\mathcal N(1,1/10)}[g'(\Lambda)^2]$ at strengths $0,10^{-4},10^{-3},10^{-2},10^{-1}$, transformed identically in every basis. Steps use the full penalized Hessian. A three-noise-level screen on four reused draws selects SD 0.6 before 40 independent draws. Five-update enhanced Hermite leads every enhanced rival, but its incremental gain over plain Hermite has adjusted lower bound $-0.00026$. Nonzero strength is selected on 14/40 draws.

The scarcity screen crosses $K=4/6$, 16/32 labels and noise SD 0.6/1.2 on six reused draws, using the first labels of the existing training mask and unchanged validation/test nodes. It adds strength 1 to the prior grid. Selection requires 10\% gains over both plain Hermite and the best enhanced rival, maximizing the smaller gain with degree/label/noise ties. No derivative candidate qualifies. Curvature replaces the derivative norm by $\mathbb E[g''(\Lambda)^2]$, equal to $100\sum_k k(k-1)\theta_k^2$ in aligned coordinates. It selects $(K,m,\sigma)=(4,16,1.2)$ after exploratory own/rival gains of 17.7\%/24.4\%, before 40 new draws. Six Bonferroni-adjusted one-sided bounds at five updates resolve the five rival comparisons but not improvement over plain Hermite (lower bound $-0.00319$). Budgets 20 and 100 are secondary.

The odd-hop screen retains odd degrees, or $(B_j-B_{K-j})/\sqrt2$ for Bernstein, giving $\lceil K/2\rceil$ coefficients per basis and the same odd subspace. Normalization is recomputed and the derivative penalty is absent. No candidate meets both 10\% thresholds, so neither independent confirmation nor the planned cosine-mismatch evaluation is conducted.

\begin{figure}[!htbp]
\centering
\includegraphics[width=\linewidth]{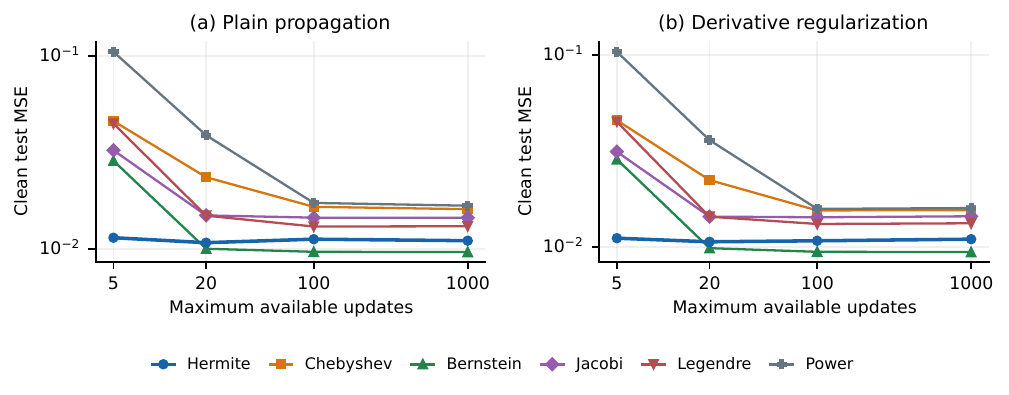}
\caption{First-derivative regularization at noise SD 0.6. Points show mean clean test MSE over 40 paired draws, selected by validation within each budget. The regularized menu includes zero penalty. Both axes are logarithmic. Hermite leads at the shortest budget, but its incremental improvement is unresolved and Bernstein becomes better with longer training.}
\label{fig:shared_prior_budgets}
\end{figure}

\begin{table}[!htbp]
\centering
\caption{Independent curvature-prior evaluation with 16 labels and noise SD 1.2. Five-update entries are mean $\pm$ sample SD across 40 independent paired draws. Enhanced entries at longer budgets are means.}
\label{tab:curvature_confirmation}
\small
\begin{tabular}{@{}lccrr@{}}
\toprule
Basis & Plain, 5 & Curvature, 5 & Curvature, 20 & Curvature, 100\\
\midrule
Hermite & $\best{0.23441}\pm0.23177$ & $\best{0.22850}\pm0.22874$ & \best{0.19830} & \best{0.19151}\\
Chebyshev & $0.32131\pm0.27501$ & $0.30756\pm0.25514$ & 0.23551 & 0.21397\\
Bernstein & $0.28927\pm0.23569$ & $0.27602\pm0.23348$ & 0.21082 & 0.19416\\
Jacobi & $0.30464\pm0.23769$ & $0.29328\pm0.23938$ & 0.23474 & 0.20784\\
Legendre & $0.31537\pm0.24678$ & $0.30523\pm0.24714$ & 0.23475 & 0.21428\\
Power & $0.34023\pm0.22880$ & $0.32909\pm0.23332$ & 0.23542 & 0.21291\\
\bottomrule
\end{tabular}
\end{table}

\FloatBarrier

The complete exploratory grids are retained in the reproducibility archive. They are summarized here to distinguish adaptive task selection from the independent evaluation endpoints.
\subsection{Classification and architectural ablations}
\label{app:classification}

\FloatBarrier
\subsubsection{Plain citation benchmarks}
\label{app:plain_real}
\label{sec:plain_real_matched}
\label{app:real_matched}
All bases use a width-32 ReLU predictor, input/hidden/prefilter dropout 0.5, predictor Adam decay 0.005 including biases, zero filter decay, unrestricted coefficients and identity initialization. Four settings cross degree 2/4 and tied predictor/filter rate 0.01/0.03. The cap is 400 updates with validation every five. Two tuning seeds select by accuracy then cross-entropy, followed by three distinct final seeds. CE-first selection is a separate sensitivity analysis. Cora uses the previously generated stratified split, while CiteSeer and PubMed use public splits. There is one split per graph.

The separate reference-moment arm fits a feature-only MLP for 300 updates at rate 0.01, selects its checkpoint on validation labels and uses its raw logits for moments. The scale floor is 0.05 with radius-six fallback (none triggered). This differs from degree-specific calibration in Table~\ref{tab:real}. The CiteSeer failure is retained. Saved states reproduce all final scores.

\begin{table}[!htbp]
\centering
\caption{Matched citation sensitivity, accuracy (\%), mean $\pm$ SD over three seeds. The first seven rows use CE-first configuration and checkpoint selection. The final row retains the separate accuracy-first reference-moment results. Primary plain results are in Table~\ref{tab:main_real}. Bold compares the six plain bases within each dataset.}
\label{tab:plain_real_ce}
\small
\begin{tabular}{@{}lccc@{}}
\toprule
Model & Cora & CiteSeer & PubMed\\
\midrule
Hermite & $79.63\pm1.35$ & $69.87\pm0.59$ & $78.37\pm0.72$\\
Chebyshev & $\best{81.57}\pm0.74$ & $71.73\pm0.12$ & $80.20\pm0.70$\\
Bernstein & $80.90\pm0.70$ & $70.13\pm0.75$ & $79.27\pm0.40$\\
Jacobi & $80.70\pm0.00$ & $69.73\pm1.69$ & $\best{80.93}\pm0.32$\\
Legendre & $80.73\pm1.17$ & $\best{71.97}\pm0.55$ & $80.83\pm0.72$\\
Power & $81.50\pm0.69$ & $71.73\pm0.15$ & $79.97\pm0.75$\\
Reference moments (CE-first) & $79.23\pm1.20$ & $31.47\pm6.74$ & $78.93\pm0.25$\\
Reference moments (accuracy-first) & $79.23\pm1.20$ & $33.07\pm8.29$ & $78.30\pm0.66$\\
\bottomrule
\end{tabular}
\end{table}

\begin{table}[!htbp]
\centering
\caption{Selected degree/learning-rate pairs for the matched citation experiment. Accuracy-first and CE-first use the same four-setting menu.}
\label{tab:plain_real_configs}
\footnotesize\setlength{\tabcolsep}{4pt}
\begin{tabular}{@{}lcccccc@{}}
\toprule
& \multicolumn{2}{c}{Cora} & \multicolumn{2}{c}{CiteSeer} & \multicolumn{2}{c}{PubMed}\\
Model & A & CE & A & CE & A & CE\\
\midrule
Hermite & 2/0.03 & 4/0.03 & 2/0.03 & 2/0.03 & 2/0.03 & 2/0.03\\
Chebyshev & 4/0.03 & 4/0.03 & 2/0.01 & 4/0.01 & 4/0.01 & 4/0.01\\
Bernstein & 4/0.03 & 4/0.03 & 4/0.03 & 4/0.03 & 4/0.01 & 2/0.03\\
Jacobi & 2/0.03 & 4/0.03 & 2/0.01 & 4/0.01 & 4/0.01 & 4/0.01\\
Legendre & 4/0.03 & 4/0.03 & 2/0.01 & 4/0.01 & 4/0.01 & 4/0.01\\
Power & 4/0.03 & 4/0.03 & 4/0.01 & 4/0.01 & 4/0.01 & 4/0.01\\
Reference moments & 4/0.03 & 4/0.03 & 4/0.03 & 4/0.03 & 4/0.01 & 2/0.03\\
\bottomrule
\end{tabular}
\end{table}

\FloatBarrier
\subsubsection{Ring classification and architectural controls}
\label{sec:shared_enhancements}
\label{app:new_classification}
A 32-case screen crosses products, stochastic blocks, geometric graphs and rings, Gaussian-shaped/uniform feature energy, and heat/high-pass/band-pass/sine teachers. Graphs have 256 nodes, eight feature channels and 40/64/152 masks. Eigenvalue quadrature masses and random signs generate spectral features. Labels threshold the normalized filtered first feature with Gaussian score-noise SD 0.3. All bases use degree four, rates 0.003/0.01/0.03, 300 Adam updates and validation every five. Accuracy-first and CE-first each select rate and checkpoint. Enhanced arms use input moments, floor 0.05 and response RMS normalization, without trained reference, radius guard or derivative penalty.

Four selected cases are evaluated on 12 new graphs and two optimizer seeds per graph. Enhanced Hermite leads the ring/uniform/wave common-basis comparison, but native JacobiConv scores higher. Gaussian-energy ring superiority is unresolved. Plain stochastic-block/uniform/heat becomes a tie with Jacobi at 89.72\%. Further fixed-scale and filter-only screens yield no primary Hermite lead.

The initialization ablation maps the native Jacobi polynomial $\sum_{k=0}^4\tanh(1)^kP_k^{(0.5,0.5)}(1-\lambda)$ into every basis. It compares direct coefficients with shared multiplicative decomposition, omitting the unused final source parameter. Predictor, rates, bias decay, zero dropout and 300 updates are shared. Source/cached predictions agree to $7.16\times10^{-7}$.

The dropout ablation shares input and prefilter masks at zero or 0.5, giving six rate/dropout settings per basis and coordinate policy. Predictor decay is 0.0005 including bias, filter decay zero, and degree/budget remain four/300. Masks are paired across arms. These interventions reuse graphs and are exploratory. Figure~\ref{fig:ring_confirm} reports the ring comparison, Table~\ref{tab:pcd_ablation} gives the initialization controls, and Figure~\ref{fig:dropout_accuracy} shows that common dropout removes the earlier Hermite lead.

\begin{figure}[!htbp]
\centering
\includegraphics[width=\linewidth]{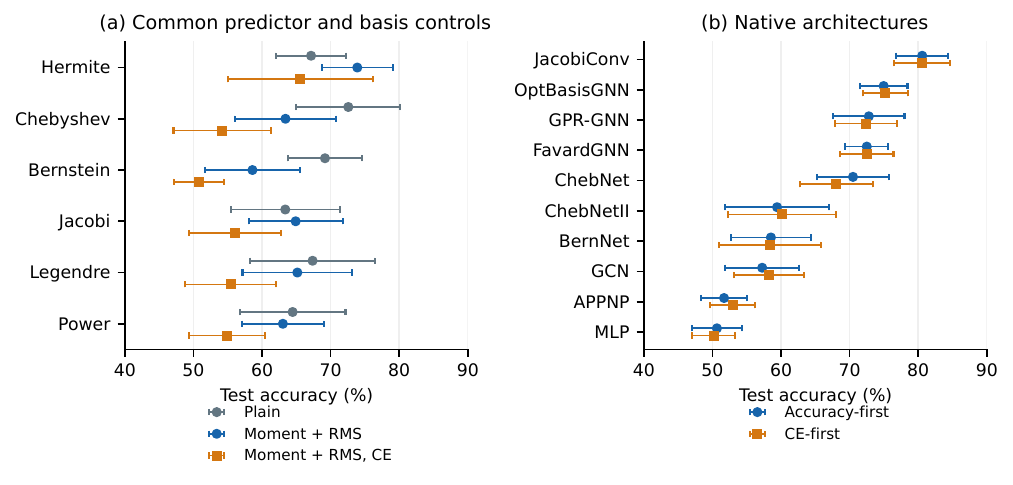}
\caption{Ring/uniform/wave classification: mean accuracy $\pm$ graph SD over 12 graphs, after averaging two optimizer seeds per graph. Left: common-predictor basis comparisons, with plain or moment/RMS coordinates. CE denotes cross-entropy-first selection. Right: native architectures on the same graphs. Their adaptively specified menus cross learning rate with dropout and differ from the common-model menus. Error bars describe graph variability, not confidence intervals.}
\label{fig:ring_confirm}
\label{fig:ring_native}
\end{figure}

\begin{table}[!htbp]
\centering
\caption{Common-basis initialization and coefficient-decomposition ablation, mean accuracy (\%) with CE-first in parentheses. All arms use no dropout. ID denotes identity initialization, J the mapped native initial polynomial, and PCD multiplicative coefficient decomposition.}
\label{tab:pcd_ablation}
\small
\begin{tabular}{@{}lrrrr@{}}
\toprule
Basis & ID, direct & J, direct & ID, PCD & J, PCD\\
\midrule
Hermite & 65.35 (60.64) & 66.72 (62.09) & 66.04 (60.64) & 63.65 (60.20)\\
Chebyshev & \best{71.08} (63.29) & 71.44 (66.72) & \best{71.82} (63.32) & 69.60 (63.29)\\
Bernstein & 70.23 (62.88) & 69.96 (64.47) & 69.71 (60.25) & 68.67 (62.01)\\
Jacobi & 67.93 (62.31) & \best{71.55} (\best{67.93}) & 67.52 (62.83) & 69.33 (63.05)\\
Legendre & 68.97 (\best{64.53}) & 68.86 (65.08) & 69.33 (\best{64.04}) & \best{70.81} (\best{65.38})\\
Power & 64.64 (58.25) & 65.21 (59.05) & 64.17 (57.89) & 66.01 (60.39)\\
\bottomrule
\end{tabular}
\end{table}

\clearpage
\begin{figure}[H]
\centering
\includegraphics[width=\linewidth]{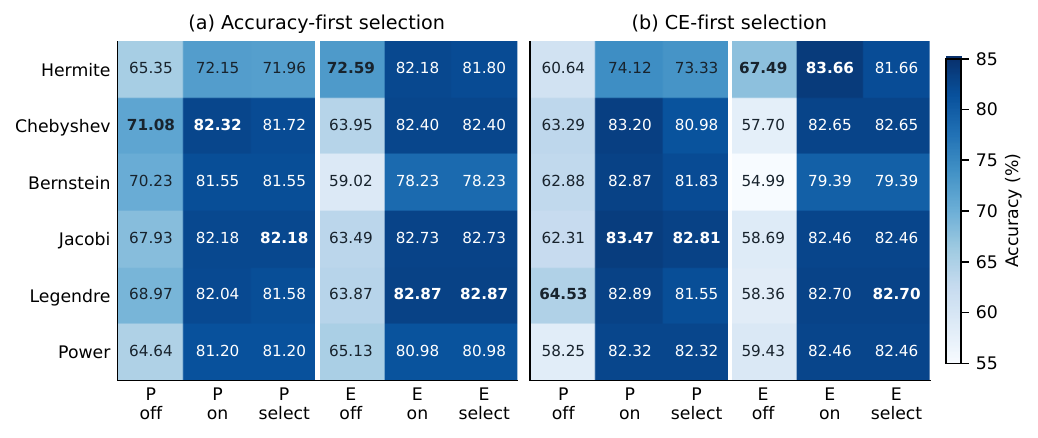}
\caption{Dropout ablation on 12 reused ring graphs. Cells give every reported mean accuracy (\%), with column maxima in bold. P/E denote plain/moment-RMS coordinates. Off/on use dropout 0/0.5, and select uses validation. A common color scale compares both checkpoint rules. The earlier Hermite lead disappears under the common dropout search.}
\label{fig:dropout_accuracy}
\label{fig:dropout_ce}
\end{figure}

\FloatBarrier
\subsubsection{Calibrated citation and synthetic neural comparisons}
\label{sec:neural_results}
\label{app:neural_tables}
\label{app:interventions}
Protocols are in Appendices~\ref{app:classification_protocol} and~\ref{app:regression_protocol}. Table~\ref{tab:real} evaluates trained-reference HermNet, not plain Algorithm~\ref{alg:filter}. JacobiConv has the largest mean on all three datasets. The limited OptBasisGNN menu omits its degree-12 PubMed preset. Synthetic classification averages three reference initializations within each of five graphs. Regression averages two optimizer seeds within each of eight evaluation graphs. Native architectures retain their own structures and penalties.

Development ablations favor affine over nonlinear predictors in all 64 tested basis/condition/penalty means, a shared capacity effect. Changing the calibration reference or extending 1,200 updates to 3,600 does not give a consistent Hermite gain. A conditionally selected positive derivative strength 0.01 increases MSE in all eight learned-regression conditions when compared with removing that penalty at otherwise fixed settings. A separate fixed-feature comparison yields MSE 0.02993 for Hermite derivative, 0.03548 for Hermite ridge and 0.03094 for a stronger Chebyshev degree prior. These reused-data ablations do not establish independent confirmation.

\begin{table}[!htbp]
\centering
\caption{Node classification accuracy (\%) under the same 16-configuration search count per method and dataset. HermNet uses trained-reference calibration. The plain fixed-coordinate model is evaluated separately. Entries are mean $\pm$ sample SD across three final optimizer seeds on one fixed split. Bold marks the highest mean, without implying statistical significance. Architecture-specific search spaces and parameter counts differ.}
\label{tab:real}
\small
\begin{tabular}{@{}lccc@{}}
\toprule
Method & Cora & CiteSeer & PubMed\\
\midrule
\texttt{MLP} & $61.40\pm0.53$ & $60.53\pm0.55$ & $70.83\pm0.40$\\
\HermitePoly{} & $83.80\pm0.26$ & $71.30\pm1.22$ & $74.83\pm1.46$\\
\ChebNetII{} & $83.60\pm0.50$ & $71.17\pm0.64$ & $73.00\pm5.40$\\
\JacobiConv{} & $\best{83.90}\pm0.30$ & $\best{72.13}\pm0.57$ & $\best{76.53}\pm0.42$\\
\FavardGNN{} & $80.20\pm1.18$ & $69.17\pm0.51$ & $74.20\pm1.67$\\
\OptBasisGNN{} & $68.17\pm3.44$ & $57.70\pm2.69$ & $69.67\pm1.15$\\
\bottomrule
\end{tabular}

\end{table}

\begin{table}[!htbp]
\centering\small
\caption{Synthetic classification summary, accuracy (\%). Each entry averages the four reported condition means for two targets and frozen/joint predictors at a fixed family and spectral-energy shape. Values are derived from the rounded condition means, and are descriptive because conditions reuse graphs. Bold compares methods within a row. Complete condition means and graph SDs are retained in the experimental archive.}
\label{tab:factorial_sbm}\label{tab:factorial_geo}
\begin{tabular}{@{}llccc@{}}
\toprule
Family & Energy & Hermite & Cheb. degree & Cheb. ridge\\
\midrule
SBM & G.4 & 42.77 & 42.73 & \best{42.83}\\
SBM & Uniform & \best{43.79} & 43.67 & 43.61\\
SBM & Bimodal & \best{41.14} & 41.01 & 40.90\\
Geometric & G.4 & \best{41.19} & 41.17 & 41.16\\
Geometric & Uniform & 42.23 & 42.32 & \best{42.33}\\
Geometric & Bimodal & \best{41.86} & 41.84 & 41.70\\
\bottomrule
\end{tabular}
\end{table}

\begin{table}[!htbp]
\centering\footnotesize\setlength{\tabcolsep}{4pt}
\caption{Learned-feature regression, clean test MSE, mean $\pm$ graph SD over eight new graphs after averaging two optimizer seeds. Matched bases share the predictor and objective. Native models retain their architectures. Bold marks the smallest column mean across all displayed procedures.}
\label{tab:regression_sbm}
\label{tab:regression_geo}
\textbf{(a) Stochastic-block graphs}\par\smallskip
\begin{tabular}{@{}lcccc@{}}
\toprule
Method & G.4 smooth & G.4 cutoff & Uniform smooth & Uniform cutoff\\
\midrule
Hermite (matched) & $0.250\pm0.101$ & $0.382\pm0.139$ & $0.199\pm0.078$ & $0.252\pm0.040$\\
Chebyshev (matched) & $0.235\pm0.054$ & $0.402\pm0.211$ & $0.255\pm0.157$ & $0.299\pm0.152$\\
Jacobi (fixed) & $0.284\pm0.076$ & $0.405\pm0.219$ & $0.157\pm0.035$ & $0.303\pm0.099$\\
Whitening & $0.210\pm0.079$ & $0.389\pm0.156$ & $0.194\pm0.128$ & $0.427\pm0.323$\\
\ChebNetII{} & $0.182\pm0.035$ & $0.289\pm0.066$ & $0.104\pm0.033$ & $0.208\pm0.045$\\
\JacobiConv{} & $\best{0.051}\pm0.012$ & $\best{0.170}\pm0.046$ & $\best{0.039}\pm0.011$ & $\best{0.122}\pm0.024$\\
\FavardGNN{} & $0.289\pm0.049$ & $0.509\pm0.087$ & $0.194\pm0.068$ & $0.334\pm0.058$\\
\OptBasisGNN{} & $0.182\pm0.033$ & $0.360\pm0.093$ & $0.141\pm0.021$ & $0.277\pm0.065$\\
\bottomrule
\end{tabular}
\par\medskip\textbf{(b) Geometric graphs}\par\smallskip
\begin{tabular}{@{}lcccc@{}}
\toprule
Method & G.4 smooth & G.4 cutoff & Uniform smooth & Uniform cutoff\\
\midrule
Hermite (matched) & $0.280\pm0.074$ & $0.354\pm0.083$ & $0.156\pm0.068$ & $0.239\pm0.061$\\
Chebyshev (matched) & $0.274\pm0.043$ & $0.357\pm0.142$ & $0.248\pm0.102$ & $0.319\pm0.092$\\
Jacobi (fixed) & $0.260\pm0.043$ & $0.321\pm0.066$ & $0.173\pm0.041$ & $0.303\pm0.089$\\
Whitening & $0.310\pm0.126$ & $0.365\pm0.125$ & $0.173\pm0.032$ & $0.328\pm0.176$\\
\ChebNetII{} & $0.249\pm0.067$ & $0.299\pm0.050$ & $0.170\pm0.071$ & $0.220\pm0.046$\\
\JacobiConv{} & $\best{0.079}\pm0.032$ & $\best{0.224}\pm0.059$ & $\best{0.056}\pm0.018$ & $\best{0.165}\pm0.030$\\
\FavardGNN{} & $0.378\pm0.101$ & $0.505\pm0.104$ & $0.238\pm0.051$ & $0.351\pm0.057$\\
\OptBasisGNN{} & $0.311\pm0.029$ & $0.343\pm0.059$ & $0.302\pm0.045$ & $0.279\pm0.099$\\
\bottomrule
\end{tabular}
\end{table}

\FloatBarrier
\subsubsection{Expanded product, citation and WebKB comparisons}
\label{sec:later_learning}
\label{app:later_tables}
Appendix~\ref{app:later_protocol} specifies the expanded searches. Results include unfavorable Hermite outcomes and distinguish common models from native architectures. WebKB uses Geom-GCN splits \citep{pei2020}. CE-first results change checkpoint selection at fixed selected configurations. Means are descriptive development results, not population rankings.

\begin{table}[!htbp]\centering\footnotesize\setlength{\tabcolsep}{4pt}
\caption{Validation-selected product-graph procedures. Accuracy percentages are mean and sample SD over eight graph/optimizer pairs per condition. A and E denote accuracy-first and cross-entropy-first checkpoint selection on the same selected configuration. These are separate from the four-channel wave experiments.}\label{tab:selected_product}
\begin{tabular}{@{}lcccc@{}}
\toprule
Method & Balanced A & Balanced E & Dominant A & Dominant E\\
\midrule
\texttt{Hermite} & $87.61\pm2.08$ & $86.37\pm2.44$ & $86.33\pm2.10$ & $86.45\pm1.86$\\
\texttt{Chebyshev} & $87.20\pm2.44$ & $86.78\pm1.78$ & $87.12\pm1.30$ & $85.96\pm1.64$\\
\texttt{Bernstein} & $86.97\pm3.01$ & $87.01\pm2.13$ & $\best{87.50}\pm1.76$ & $\best{86.94}\pm1.25$\\
\texttt{Jacobi} & $86.78\pm2.89$ & $86.67\pm2.49$ & $86.71\pm1.87$ & $85.77\pm1.46$\\
\texttt{Legendre} & $\best{87.76}\pm1.63$ & $\best{87.20}\pm1.97$ & $86.56\pm1.88$ & $85.77\pm1.46$\\
\texttt{Power} & $87.69\pm1.89$ & $86.90\pm2.21$ & $87.09\pm2.02$ & $86.41\pm1.65$\\
\texttt{QR} & $86.33\pm1.80$ & $86.63\pm1.91$ & $86.78\pm2.14$ & $86.78\pm1.60$\\
\texttt{ChebNet} & $83.62\pm2.27$ & $84.22\pm2.16$ & $83.28\pm1.55$ & $84.22\pm2.42$\\
\texttt{ChebNetII} & $85.13\pm1.96$ & $83.28\pm2.80$ & $83.73\pm2.01$ & $84.34\pm1.70$\\
\texttt{BernNet} & $86.11\pm1.57$ & $84.11\pm4.04$ & $84.98\pm1.95$ & $83.81\pm2.18$\\
\texttt{JacobiConv} & $86.71\pm2.94$ & $87.05\pm2.53$ & $85.54\pm1.97$ & $86.22\pm2.07$\\
\texttt{GPR-GNN} & $85.13\pm2.10$ & $83.89\pm2.15$ & $84.64\pm2.04$ & $85.02\pm1.95$\\
\texttt{FavardGNN} & $83.81\pm1.11$ & $82.98\pm0.97$ & $82.61\pm1.53$ & $83.55\pm2.46$\\
\texttt{OptBasisGNN} & $85.35\pm2.22$ & $85.47\pm2.07$ & $83.62\pm2.76$ & $83.28\pm1.72$\\
\texttt{APPNP} & $85.54\pm1.95$ & $85.39\pm1.80$ & $84.64\pm0.94$ & $84.90\pm1.17$\\
\texttt{GCN} & $77.60\pm3.26$ & $77.07\pm2.45$ & $75.90\pm3.32$ & $76.77\pm2.70$\\
\texttt{MLP} & $76.20\pm2.96$ & $78.50\pm2.21$ & $77.94\pm2.21$ & $78.35\pm1.52$\\
\bottomrule
\end{tabular}\end{table}

\begin{table}[!htbp]\centering\footnotesize\setlength{\tabcolsep}{4pt}
\caption{Validation-selected real-data procedures, accuracy-first checkpoints. Entries are accuracy percentages with sample SD. CiteSeer uses three optimizer seeds on one partition. WebKB uses ten overlapping split means, each averaging two seeds. The uncertainty sources are different. QR was not included in this CiteSeer comparison.}\label{tab:selected_real_accuracy}
\begin{tabular}{@{}lcccc@{}}
\toprule
Method & CiteSeer & Cornell & Texas & Wisconsin\\
\midrule
\texttt{Hermite} & $67.03\pm0.06$ & $66.35\pm8.65$ & $74.86\pm9.67$ & $80.59\pm2.16$\\
\texttt{Chebyshev} & $65.53\pm3.10$ & $69.46\pm5.67$ & $79.86\pm6.45$ & $80.88\pm4.72$\\
\texttt{Bernstein} & $58.03\pm1.36$ & $66.35\pm7.68$ & $78.11\pm6.68$ & $75.98\pm3.28$\\
\texttt{Jacobi} & $69.07\pm0.45$ & $70.95\pm7.49$ & $78.51\pm6.87$ & $80.10\pm3.30$\\
\texttt{Legendre} & $68.60\pm0.17$ & $72.97\pm5.02$ & $78.78\pm6.68$ & $79.80\pm4.09$\\
\texttt{Power} & $70.90\pm0.10$ & $72.97\pm6.98$ & $77.70\pm8.79$ & $81.57\pm4.08$\\
\texttt{QR} & --- & $65.27\pm6.62$ & $77.30\pm7.23$ & $80.69\pm4.26$\\
\texttt{ChebNet} & $69.77\pm0.72$ & $62.57\pm8.50$ & $76.22\pm5.45$ & $75.10\pm5.39$\\
\texttt{ChebNetII} & $69.60\pm0.26$ & $73.78\pm4.09$ & $\best{81.89}\pm4.04$ & $80.29\pm3.77$\\
\texttt{BernNet} & $\best{71.77}\pm0.83$ & $73.51\pm6.69$ & $80.95\pm4.34$ & $83.73\pm2.45$\\
\texttt{JacobiConv} & $70.37\pm0.85$ & $65.68\pm6.28$ & $79.19\pm7.54$ & $80.59\pm4.73$\\
\texttt{GPR-GNN} & $71.20\pm0.61$ & $72.03\pm9.47$ & $81.62\pm5.30$ & $83.82\pm3.82$\\
\texttt{FavardGNN} & $67.63\pm0.38$ & $68.92\pm7.59$ & $75.95\pm3.92$ & $79.41\pm3.17$\\
\texttt{OptBasisGNN} & $53.00\pm3.00$ & $62.97\pm6.75$ & $75.54\pm3.52$ & $76.27\pm5.29$\\
\texttt{APPNP} & $71.03\pm0.35$ & $49.86\pm5.42$ & $63.24\pm6.04$ & $61.37\pm2.93$\\
\texttt{GCN} & $\best{71.77}\pm0.95$ & $50.95\pm8.23$ & $64.86\pm4.37$ & $61.37\pm6.36$\\
\texttt{MLP} & $55.80\pm1.23$ & $\best{74.46}\pm3.74$ & $81.62\pm4.04$ & $\best{85.10}\pm4.00$\\
\bottomrule
\end{tabular}\end{table}

\begin{table}[!htbp]\centering\footnotesize\setlength{\tabcolsep}{4pt}
\caption{CE-first sensitivity of the validation-selected real-data procedures, accuracy (\%), mean $\pm$ SD. The selected configurations are held fixed. The highest non-Hermite mean is a descriptive comparison, not a test-based selection rule. The complete roster appears in the experimental archive. Evaluation units match Table~\ref{tab:selected_real_accuracy}.}\label{tab:selected_real_ce}
\begin{tabular}{@{}lclc@{}}
\toprule
Dataset & Hermite & Highest observed rival & Rival accuracy\\
\midrule
CiteSeer & $64.20\pm0.26$ & \texttt{APPNP} & $\best{71.80}\pm0.26$\\
Cornell & $65.41\pm8.14$ & \texttt{MLP} & $\best{74.05}\pm3.48$\\
Texas & $74.59\pm7.42$ & \texttt{GPR-GNN} & $\best{82.03}\pm4.13$\\
Wisconsin & $80.49\pm3.62$ & \texttt{MLP} & $\best{84.61}\pm4.39$\\
\bottomrule
\end{tabular}\end{table}

\begin{table}[!htbp]\centering\footnotesize\setlength{\tabcolsep}{4pt}
\caption{Prespecified WebKB common-model references, primary accuracy percentages. U is the nonlinear, no-dropout reference with fixed coordinates and raw responses. T uses the linear-predictor/dropout configuration transferred from the synthetic task. Each entry is mean and sample SD over ten split means. These references are included when describing the strongest observed rival.}\label{tab:webkb_references}
\begin{tabular}{@{}lcccccc@{}}
\toprule
Family & Cornell U & Cornell T & Texas U & Texas T & Wisconsin U & Wisconsin T\\
\midrule
\texttt{Hermite} & $67.57\pm6.01$ & $63.11\pm6.50$ & $78.92\pm5.30$ & $71.49\pm7.97$ & $77.25\pm4.28$ & $74.71\pm4.66$\\
\texttt{Chebyshev} & $67.30\pm4.72$ & $66.62\pm6.05$ & $80.00\pm5.65$ & $75.14\pm5.67$ & $81.37\pm3.81$ & $76.08\pm4.69$\\
\texttt{Bernstein} & $73.38\pm4.64$ & $62.03\pm6.63$ & $\best{82.30}\pm7.19$ & $76.22\pm6.02$ & $82.35\pm3.87$ & $72.06\pm6.78$\\
\texttt{Jacobi} & $70.68\pm5.56$ & $67.70\pm2.25$ & $79.46\pm4.58$ & $76.49\pm6.66$ & $82.06\pm4.11$ & $75.49\pm5.45$\\
\texttt{Legendre} & $70.68\pm4.18$ & $67.97\pm3.06$ & $80.00\pm5.58$ & $76.08\pm6.92$ & $83.53\pm2.73$ & $75.39\pm5.28$\\
\texttt{Power} & $\best{75.27}\pm5.14$ & $\best{68.65}\pm3.54$ & $81.76\pm4.69$ & $\best{77.70}\pm5.67$ & $\best{84.31}\pm4.26$ & $\best{76.86}\pm4.51$\\
\texttt{QR} & $65.27\pm7.70$ & $62.30\pm5.23$ & $76.76\pm5.83$ & $76.35\pm8.38$ & $79.71\pm2.85$ & $74.90\pm2.54$\\
\bottomrule
\end{tabular}\end{table}

\FloatBarrier
\clearpage
\subsubsection{Channel-specific filters}
\label{sec:channel_results}
Full channel matrices enlarge every basis's architecture. Tables~\ref{tab:channel_complete} and~\ref{tab:channel_effects} report complete scores and paired full-minus-factorized differences. Some references share a configuration and trajectory. They are not independent fits. Several bases improve, and the selection criterion changes rankings.

\begin{table}[H]
\centering\footnotesize\setlength{\tabcolsep}{4pt}
\caption{Four-channel wave classification: full and factorized common models and the native roster. A/E denote accuracy-first/CE-first checkpoint rules. Accuracy (\%), mean $\pm$ SD across eight graph means, each averaging two optimizers. Common arms share cached responses and dropout masks across orders. Native configurations transfer from the diffusion task. ChebConv-linear retains its native initialization and regularization. Bold marks the largest column mean. Additional fixed, raw and tuned common references are retained in the experimental archive.}\label{tab:channel_complete}
\begin{tabular}{@{}lccc@{}}
\toprule
Procedure & A accuracy & E accuracy & Parameters\\
\midrule
\texttt{Hermite} (factorized) & $75.26\pm1.63$ & $75.40\pm2.03$ & 28\\
\texttt{Hermite} (full) & $\best{78.93}\pm2.04$ & $78.60\pm1.88$ & 82\\
\texttt{Chebyshev} (factorized) & $74.89\pm2.19$ & $75.23\pm2.51$ & 28\\
\texttt{Chebyshev} (full) & $78.35\pm2.55$ & $78.77\pm1.94$ & 82\\
\texttt{Bernstein} (factorized) & $74.64\pm3.14$ & $74.96\pm2.97$ & 28\\
\texttt{Bernstein} (full) & $77.60\pm2.86$ & $78.11\pm2.09$ & 82\\
\texttt{Jacobi} (factorized) & $75.21\pm1.71$ & $75.73\pm2.28$ & 28\\
\texttt{Jacobi} (full) & $78.60\pm2.38$ & $78.75\pm2.10$ & 82\\
\texttt{Legendre} (factorized) & $75.04\pm1.90$ & $75.79\pm2.28$ & 28\\
\texttt{Legendre} (full) & $78.60\pm2.48$ & $78.95\pm2.13$ & 82\\
\texttt{Power} (factorized) & $75.56\pm1.37$ & $76.26\pm2.09$ & 28\\
\texttt{Power} (full) & $78.71\pm2.31$ & $79.07\pm1.96$ & 82\\
\texttt{QR} (factorized) & $74.68\pm1.92$ & $75.90\pm2.45$ & 28\\
\texttt{QR} (full) & $77.94\pm1.78$ & $78.82\pm2.05$ & 82\\
\texttt{ChebNet} (tuned) & $73.64\pm2.18$ & $73.91\pm2.23$ & 994\\
\texttt{ChebNetII} (tuned) & $63.97\pm3.44$ & $64.44\pm2.57$ & 357\\
\texttt{BernNet} (tuned) & $64.85\pm2.05$ & $64.46\pm2.23$ & 361\\
\texttt{JacobiConv} (tuned) & $75.66\pm1.65$ & $72.18\pm4.40$ & 33\\
\texttt{GPR-GNN} (tuned) & $63.91\pm2.20$ & $64.36\pm3.39$ & 357\\
\texttt{FavardGNN} (tuned) & $68.64\pm2.59$ & $70.59\pm2.97$ & 642\\
\texttt{OptBasisGNN} (tuned) & $73.57\pm3.14$ & $74.79\pm2.83$ & 514\\
\texttt{APPNP} (tuned) & $60.94\pm2.65$ & $60.81\pm3.28$ & 354\\
\texttt{GCN} (tuned) & $56.17\pm3.24$ & $56.80\pm3.09$ & 354\\
\texttt{MLP} (tuned) & $63.84\pm3.07$ & $64.19\pm3.03$ & 178\\
\texttt{ChebConv-linear} (fixed) & $78.73\pm1.98$ & $\best{79.39}\pm2.79$ & 82\\
\bottomrule
\end{tabular}
\end{table}

\begin{table}[H]\centering\footnotesize\setlength{\tabcolsep}{4pt}
\caption{Shared effect of replacing the factorized channel filters by unrestricted full channel filters, in percentage points. SD is over eight paired graph means. Positive counts refer to the primary accuracy-first rule.}\label{tab:channel_effects}
\begin{tabular}{@{}lccc@{}}
\toprule
Family & A effect & Positive & E effect\\
\midrule
\texttt{Hermite} & $\best{3.67}\pm1.45$ & 8/8 & $3.20\pm1.62$\\
\texttt{Chebyshev} & $3.46\pm2.24$ & 8/8 & $\best{3.54}\pm2.30$\\
\texttt{Bernstein} & $2.96\pm3.62$ & 6/8 & $3.14\pm2.03$\\
\texttt{Jacobi} & $3.39\pm2.13$ & 8/8 & $3.01\pm2.36$\\
\texttt{Legendre} & $3.56\pm2.19$ & 8/8 & $3.16\pm2.41$\\
\texttt{Power} & $3.14\pm1.98$ & 8/8 & $2.80\pm2.30$\\
\texttt{QR} & $3.26\pm1.27$ & 8/8 & $2.92\pm2.12$\\
\bottomrule
\end{tabular}\end{table}

\FloatBarrier
\subsection{Conditioning and regularization}
\label{app:results}

\FloatBarrier
\subsubsection{Coordinate conditioning}
\label{sec:coordinate_results}
\label{app:coordinate_tables}
Appendix~\ref{app:coordinate_protocol} fixes equivalent objectives and exact solutions. Counts are medians with successes out of 32 trajectories sharing eight graphs. Counts above 2,000 are censored. Whitening remains fixed under mismatch. Narrow energy and degree two favor Hermite conditioning, but no width/degree pair satisfies every stability, optimization and adequacy criterion in both graph families. Figure~\ref{fig:coordinates} and Tables~\ref{tab:coord_stable}--\ref{tab:coordinate_criteria} report these diagnostics.

\begin{figure}[!htbp]
\centering
\includegraphics[width=\linewidth]{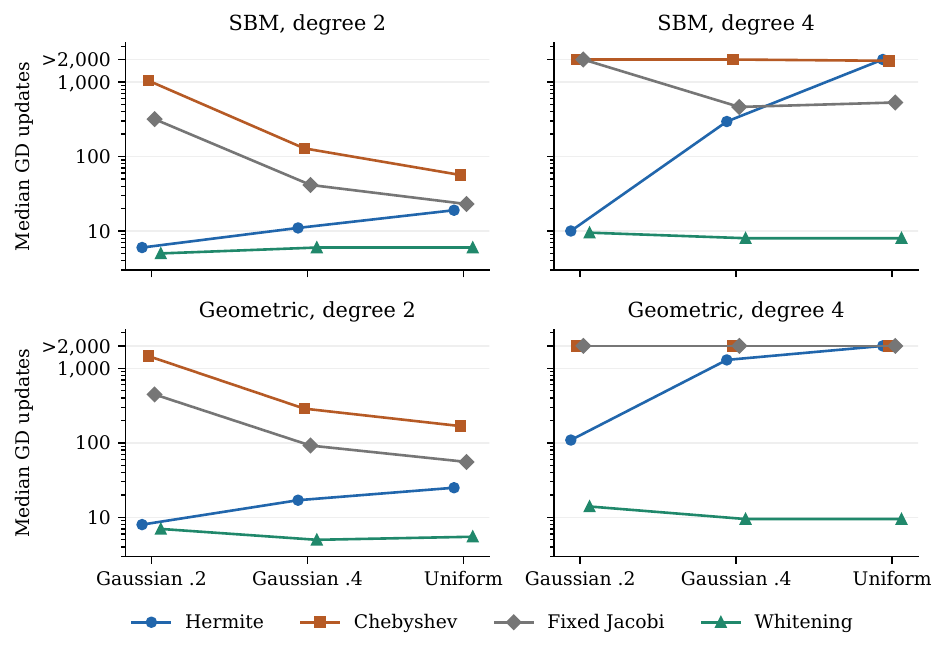}
\caption{Coordinate optimization with the reference equal to the filtered input ($\rho=0$). Each point is the median over eight graphs, two label counts, and two target responses, giving 32 paired trajectories. Counts above 2,000 denote trajectories that did not reach the objective-gap tolerance within the budget and enter the median as 2,001. The vertical axis is logarithmic. Narrow spectral energy favors Hermite coordinates, while whitening remains competitive. The figure reports iteration counts without measuring total runtime.}
\label{fig:coordinates}
\end{figure}

\begin{table}[!htbp]
\centering\footnotesize\setlength{\tabcolsep}{4pt}
\caption{Equivalent-coordinate optimization. Median gradient-descent updates with successes out of 32 trajectories (eight shared graphs). Counts above 2,000 are censored. G.2/G.4 denote Gaussian signal-energy widths. Whitening is fixed under mismatch. Bold marks the smallest finite count within each row. Fully censored rows are not ranked.}
\label{tab:coord_stable}
\label{tab:coord_mismatch}
\textbf{(a) Matched input and reference, $\rho=0$}\par\smallskip
\begin{tabular}{@{}llrcccc@{}}
\toprule
Family & Energy & $K$ & Hermite & Chebyshev & Jacobi & Whitening\\
\midrule
SBM & G.2 & 2 & $6$ (32/32) & $1042.5$ (32/32) & $317$ (32/32) & $\best{5}$ (32/32)\\
SBM & G.2 & 4 & $10$ (32/32) & $>2000$ (0/32) & $>2000$ (1/32) & $\best{9.5}$ (32/32)\\
SBM & G.4 & 2 & $11$ (32/32) & $128.5$ (32/32) & $41.5$ (32/32) & $\best{6}$ (32/32)\\
SBM & G.4 & 4 & $294.5$ (32/32) & $>2000$ (6/32) & $462$ (32/32) & $\best{8}$ (32/32)\\
SBM & Uniform & 2 & $19$ (32/32) & $56.5$ (32/32) & $23$ (32/32) & $\best{6}$ (32/32)\\
SBM & Uniform & 4 & $>2000$ (15/32) & $1922$ (20/32) & $530.5$ (32/32) & $\best{8}$ (32/32)\\
Geometric & G.2 & 2 & $8$ (32/32) & $1462$ (24/32) & $446.5$ (32/32) & $\best{7}$ (32/32)\\
Geometric & G.2 & 4 & $109$ (32/32) & $>2000$ (0/32) & $>2000$ (0/32) & $\best{14}$ (32/32)\\
Geometric & G.4 & 2 & $17$ (32/32) & $288$ (32/32) & $92$ (32/32) & $\best{5}$ (32/32)\\
Geometric & G.4 & 4 & $1297.5$ (27/32) & $>2000$ (0/32) & $>2000$ (0/32) & $\best{9.5}$ (32/32)\\
Geometric & Uniform & 2 & $25$ (32/32) & $168$ (32/32) & $55.5$ (32/32) & $\best{5.5}$ (32/32)\\
Geometric & Uniform & 4 & $>2000$ (3/32) & $>2000$ (0/32) & $>2000$ (1/32) & $\best{9.5}$ (32/32)\\
\bottomrule
\end{tabular}
\par\medskip\textbf{(b) Input mismatch, $\rho=0.5$}\par\smallskip
\begin{tabular}{@{}llrcccc@{}}
\toprule
Family & Energy & $K$ & Hermite & Chebyshev & Jacobi & Whitening\\
\midrule
SBM & G.2 & 2 & $50$ (32/32) & $176.5$ (32/32) & $52.5$ (32/32) & $\best{49.5}$ (32/32)\\
SBM & G.2 & 4 & $1178.5$ (30/32) & $>2000$ (1/32) & $\best{819.5}$ (32/32) & $1551$ (26/32)\\
SBM & G.4 & 2 & $\best{15}$ (32/32) & $91.5$ (32/32) & $32$ (32/32) & $16$ (32/32)\\
SBM & G.4 & 4 & $570.5$ (32/32) & $>2000$ (6/32) & $535$ (32/32) & $\best{45}$ (32/32)\\
SBM & Uniform & 2 & $29$ (32/32) & $68$ (32/32) & $31$ (32/32) & $\best{20}$ (32/32)\\
SBM & Uniform & 4 & $>2000$ (0/32) & $>2000$ (12/32) & $875$ (31/32) & $\best{33}$ (32/32)\\
Geometric & G.2 & 2 & $\best{55.5}$ (32/32) & $531.5$ (32/32) & $164$ (32/32) & $73.5$ (32/32)\\
Geometric & G.2 & 4 & $>2000$ (3/32) & $>2000$ (1/32) & $>2000$ (0/32) & $>2000$ (9/32)\\
Geometric & G.4 & 2 & $41.5$ (32/32) & $292.5$ (32/32) & $92$ (32/32) & $\best{19}$ (32/32)\\
Geometric & G.4 & 4 & $>2000$ (6/32) & $>2000$ (0/32) & $>2000$ (0/32) & $\best{40.5}$ (32/32)\\
Geometric & Uniform & 2 & $39$ (32/32) & $214$ (32/32) & $81.5$ (32/32) & $\best{18}$ (32/32)\\
Geometric & Uniform & 4 & $>2000$ (2/32) & $>2000$ (0/32) & $>2000$ (1/32) & $\best{28.5}$ (32/32)\\
\bottomrule
\end{tabular}
\end{table}

\begin{table}[!htbp]\centering\small
\caption{Conditioning, mismatch, and degree adequacy for all Gaussian width/degree choices. Stable condition pools graphs and label counts. Mismatch ratio is the median paired perturbed-to-stable condition ratio. Validation ratio compares a degree's exact noisy-validation MSE with the better degree at the same width, pooling both targets and mismatch levels. Bold compares degrees within each family and width. Lower diagnostic values are better.}
\label{tab:coordinate_criteria}
\begin{tabular}{@{}lrrrrr@{}}
\toprule
Family & Width & $K$ & Stable $\kappa_H$ & Mismatch ratio & Validation ratio\\
\midrule
SBM & .2 & 2 & \best{1.46} & \best{6.57} & 1.046\\
Geometric & .2 & 2 & \best{2.10} & \best{4.47} & 1.081\\
SBM & .2 & 4 & 2.61 & 74.66 & \best{1.000}\\
Geometric & .2 & 4 & 46.15 & 26.68 & \best{1.000}\\
SBM & .4 & 2 & \best{2.50} & \best{1.30} & 1.096\\
Geometric & .4 & 2 & \best{4.08} & \best{1.59} & 1.071\\
SBM & .4 & 4 & 63.61 & 1.81 & \best{1.000}\\
Geometric & .4 & 4 & 361.88 & 2.15 & \best{1.000}\\
\bottomrule
\end{tabular}

\end{table}

\FloatBarrier
\subsubsection{Prior support and representation floors}
\label{sec:prior_results}
\label{app:prior_tables}
\label{sec:error_results}
Table~\ref{tab:priors} reports the primary prior comparison, and Figure~\ref{fig:decomposition} separates representation and fitting errors. No primary bounded-Gaussian comparison meets both thresholds (0.005 absolute and 5\% relative MSE gain). Three of 160 remaining contrasts exceed the harm threshold. All occur in stochastic-block/G.4/cutoff/30-label/$\rho=0$: bounded-prior MSE exceeds Full, Uniform and Chebyshev by 0.007495, 0.007295 and 0.007531. The selected degree is four for all procedures there. Equivalent Hermite/Chebyshev predictions agree to $4.76\times10^{-14}$. No fits fail or selected strengths reach the upper grid edge.

The complete 48-condition prior factorial, including every mean and SD, is retained in the experimental archive. Across 48 conditions, Full/Bounded/Uniform/Chebyshev/Zero select degree two in 10/12/12/10/13 cases. The first four select zero roughness in 21/20/19/19. Shared zero objectives are not independent replications. Fixed-degree primary means also lack a uniform bounded-prior advantage.

The post hoc oracle decomposition uses saved predictions and noiseless evaluation targets without changing fits. The representation floor accounts for 80.86\%/85.32\% of primary bounded-prior loss on stochastic-block/geometric graphs, rising to 84.89\%/88.55\% under mismatch. Without mismatch, total MSE is about 0.0160/0.0177. Small selected-degree gains over zero roughness combine lower floors with higher within-space errors: $0.000036=0.001560-0.001524$ and $0.001914=0.002916-0.001002$. They do not establish better estimation in a common space. Degree-four oracle improvements of roughly 13.8--19.3\% remain algebraically possible, without a guarantee of attainability from the available labels.

\begin{table}[!htbp]
\centering
\caption{Noiseless test MSE for the primary derivative-prior comparison. Each entry averages the four label-count/mismatch conditions within each graph, then reports mean $\pm$ sample SD over 12 graphs. Degrees and penalty strengths are selected by noisy-validation MSE. All procedures include the same response ridge. Lower values are better.}
\label{tab:priors}
\small
\begin{tabular}{@{}lcc@{}}
\toprule
Derivative prior & SBM & Geometric\\
\midrule
Full Gaussian & $0.1689\pm0.0386$ & $0.1904\pm0.0492$\\
Bounded Gaussian & $\best{0.1686}\pm0.0382$ & $0.1917\pm0.0499$\\
Uniform derivative & $\best{0.1686}\pm0.0384$ & $\best{0.1879}\pm0.0470$\\
Chebyshev degree & $\best{0.1686}\pm0.0385$ & $0.1885\pm0.0476$\\
Zero roughness & $0.1687\pm0.0385$ & $0.1936\pm0.0485$\\
\bottomrule
\end{tabular}

\end{table}

\begin{figure}[!htbp]
\centering
\includegraphics[width=.82\linewidth]{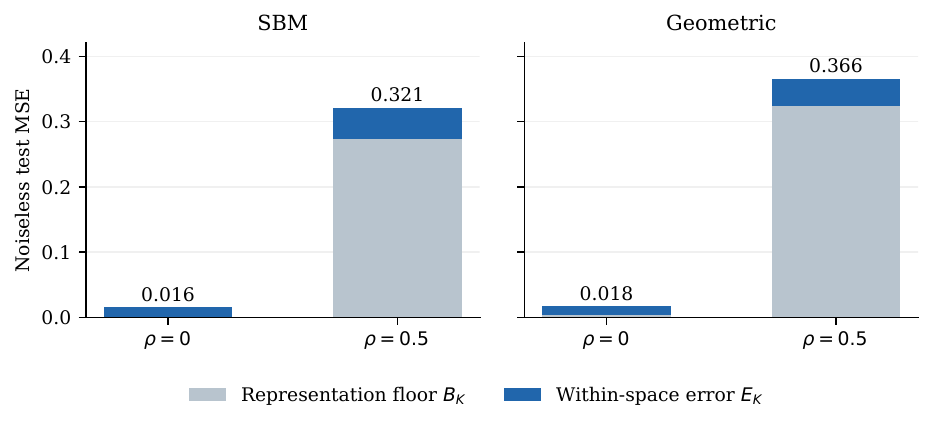}
\caption{Error decomposition for the bounded Gaussian prior with width-$0.4$ energy and a smooth target. Bars average the two label counts within each graph and then 12 graphs per family. The input mismatch changes the observed signal while leaving the target fixed. Most error at $\rho=0.5$ is the representation floor of that input's polynomial responses. Numbers above bars are total MSE. The oracle projection uses evaluation targets and is not available to the fitted learner.}
\label{fig:decomposition}
\end{figure}

\FloatBarrier
\subsection{Precision, sparsity and computational cost}
\label{sec:representation_diagnostics}
Appendix~\ref{app:constrained_protocol} defines the oracle diagnostics. Table~\ref{tab:precision_primary} gives the primary results and Figure~\ref{fig:precision_sensitivity} their sensitivity and cost. Degree-four coarse-rounding gains are conditional: fixed Legendre slightly improves wave distortion (0.2905\% versus Hermite 0.2930\%), and moment Legendre improves smooth distortion (0.3740\% versus 0.3782\%) and analytic sign accuracy by 0.0814 points. Higher degree and dominant weights reverse aggregate orderings. Sparse Hermite often improves on sparse QR, but standard-coordinate controls perform better throughout the tested settings except numerical ties. These constrained dictionaries differ despite spanning the same unrestricted polynomial space.

CPU times use prepared inputs and a fixed update count, not a common accuracy target. Repetition variability and the small setup fraction limit inferences about calibration cost. The results do not show a Hermite total-time advantage.

\begin{table}[!htbp]
\centering\footnotesize\setlength{\tabcolsep}{4pt}
\caption{Primary oracle diagnostics on eight balanced graphs (mean $\pm$ SD). Rounding uses degree four, four bits and one scale, averaged over three target groups. Sparse projection uses degree eight, moderate frequencies and three terms. M/F: moment/fixed coordinates. Error, distortion and accuracy are percentages, and loss is in percentage points. Bold marks column optima. $\dagger$ identifies the Hermite rounding result. These are oracle measurements.}
\label{tab:precision_primary}\label{tab:sparse_primary}
\begin{tabular}{@{}lcccc@{}}
\toprule
 & \multicolumn{2}{c}{Coefficient rounding} & \multicolumn{2}{c}{Sparse projection}\\
\cmidrule(lr){2-3}\cmidrule(l){4-5}
Representation & Distortion & Oracle loss & Error & Oracle accuracy\\
\midrule
\texttt{Hermite} (M) & $\key{0.3284}\pm0.0369$ & $\key{0.1417}\pm0.0091$ & $25.57\pm0.58$ & $81.57\pm0.19$\\
\texttt{Chebyshev} (M) & $0.9170\pm0.2363$ & $0.3171\pm0.0621$ & $44.55\pm0.22$ & $75.70\pm0.07$\\
\texttt{Bernstein} (M) & $10.3032\pm3.0249$ & $3.6358\pm0.4293$ & $84.80\pm0.29$ & $62.35\pm0.12$\\
\texttt{Jacobi} (M) & $0.9988\pm0.0284$ & $0.2849\pm0.0092$ & $48.96\pm0.20$ & $74.38\pm0.06$\\
\texttt{Legendre} (M) & $0.9294\pm0.0546$ & $0.3492\pm0.0076$ & $47.58\pm0.21$ & $74.79\pm0.06$\\
\texttt{Power} (M) & $1.1265\pm0.1040$ & $0.3490\pm0.0286$ & $50.80\pm0.03$ & $73.82\pm0.01$\\
\texttt{Hermite} (F) & $116.2767\pm131.9582$ & $18.1592\pm0.8460$ & $30.97\pm0.25$ & $79.84\pm0.08$\\
\texttt{Chebyshev} (F) & $2.4314\pm0.7517$ & $0.7226\pm0.1728$ & $\best{5.48}\pm0.03$ & $\best{89.22}\pm0.01$\\
\texttt{Bernstein} (F) & $5.8328\pm1.8627$ & $1.9140\pm0.3507$ & $11.12\pm0.10$ & $86.75\pm0.04$\\
\texttt{Jacobi} (F) & $0.7541\pm0.1386$ & $0.2744\pm0.0261$ & $11.59\pm0.09$ & $86.56\pm0.03$\\
\texttt{Legendre} (F) & $0.9375\pm0.0963$ & $0.2728\pm0.0124$ & $7.59\pm0.09$ & $88.25\pm0.04$\\
\texttt{Power} (F) & $1.1265\pm0.1040$ & $0.3490\pm0.0286$ & $50.80\pm0.03$ & $73.82\pm0.01$\\
\texttt{QR} & $0.4457\pm0.0218$ & $0.2004\pm0.0057$ & $34.73\pm0.20$ & $78.67\pm0.06$\\
\bottomrule
\end{tabular}
\end{table}

\begin{figure}[!htbp]
\centering
\includegraphics[width=0.96\linewidth]{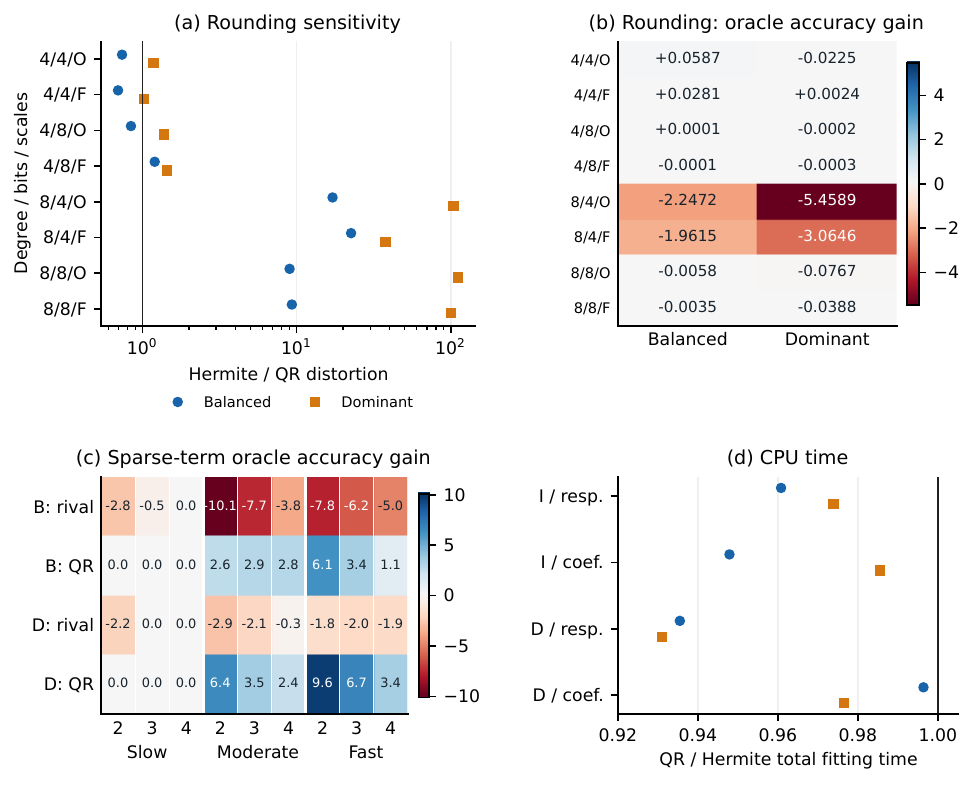}
\caption{Numerical sensitivity and CPU cost. (a) Hermite/QR rounding distortion, with O/F denoting one/four scales. (b) Hermite-minus-QR oracle accuracy (pp). (c) Hermite-minus-comparator sparse accuracy (pp), against the largest-mean rival or QR. B/D denote balanced/dominant graphs, and columns give retained terms. Positive differences favor Hermite. (d) Median paired QR/Hermite fitting-time ratios across eight graphs. I/D denote input/dropout calibration, and resp./coef. denote transform placement. Every ratio is below one. QR setup shares are 0.250\%--0.457\%. Exact values and comparator identities remain in the numerical archive.}
\label{fig:precision_sensitivity}
\label{fig:sparse_sensitivity}
\label{fig:total_cost}
\end{figure}

\FloatBarrier
\end{document}